\documentclass[12pt]{article}

\usepackage{myarticle}
\usepackage{array}
\usepackage{tabularx}

\usepackage{mycommands}

\newcommand{\EE}{\mathcal E}
\newcommand{\lift}{\ell}
\newcommand{\scaledlift}{\ell_s}
\newcommand{\ilift}{\ell^\dagger}

\newcommand{\loss}{\mathcal L}
\newcommand{\net}{\mathcal N}

\newcommand{\lI}{\underline{t}}
\newcommand{\rI}{\overline{t}}

\usepackage{pgfplots}
\pgfplotsset{
  tick label style={font=\footnotesize},
  label style={font=\footnotesize},
  legend style={font=\footnotesize}
}

\graphicspath{{./figures/}}

\TITLE{Lifting Layers: Analysis and Applications}


\begin{document}

\let\thefootnotetmp\thefootnote
\renewcommand{\thefootnote}{*}
\thispagestyle{empty}
\begin{center}
\vspace*{0.03\paperheight}
{\Large\bf\color{black}\theTITLE}\\
\bigskip
\bigskip
\bigskip
{\large Peter Ochs$^\text{\ref{note1}}$ $^\dagger$, %
        Tim Meinhardt$^\ddagger$, %
        Laura Leal-Taixe$^\ddagger$, %
        Michael Moeller\footnote{\label{note1}These authors have equally contributed.}$^\sharp$, \\ \bigskip %
{\small
$^\sharp$~University of Siegen, Siegen, Germany \\
$^\dagger$~Saarland University, Saarbr\"{u}cken, Germany \\
$^\ddagger$~TU Munich, Munich, Germany\\
}
}
\end{center}
\bigskip

\let\thefootnote\thefootnotetmp
\setcounter{footnote}{0}

\begin{abstract}
The great advances of learning-based approaches in image processing and computer vision are largely based on deeply nested networks that compose linear transfer functions with suitable non-linearities. Interestingly, the most frequently used non-linearities in imaging applications (variants of the rectified linear unit) are uncommon in low dimensional approximation problems. In this paper we propose a novel non-linear transfer function, called \textit{lifting}, which is motivated from a related technique in convex optimization. A \textit{lifting layer} increases the dimensionality of the input, naturally yields a linear spline when combined with a fully connected layer, and therefore closes the gap between low and high dimensional approximation problems. Moreover, applying the lifting operation to the loss layer of the network allows us to handle non-convex and flat (zero-gradient) cost functions. We analyze the proposed lifting theoretically, exemplify interesting properties in synthetic experiments and demonstrate its effectiveness in deep learning approaches to image classification and denoising. 
\end{abstract}

\makekeywords

\section{Introduction}

Deep Learning has seen a tremendous success within the last 10 years improving the state-of-the-art in almost all computer vision and image processing tasks significantly. While one of the main explanations for this success is the replacement of handcrafted methods and features with data-driven approaches, the architectures of successful networks remain handcrafted and difficult to interpret. 

The use of some common building blocks, such as convolutions, in imaging tasks is intuitive as they establish translational invariance. The composition of linear transfer functions with non-linearities is a natural way to achieve a simple but expressive representation, but the choice of non-linearity is less intuitive: Starting from biologically motivated step functions or their smooth approximations by sigmoids, researchers have turned to rectified linear units (ReLUs),
\begin{align}
\label{eq:relu}
 \sigma(x) = \max(x,0) 
 \end{align}
to avoid the optimization-based problem of a vanishing gradient. The derivative of a ReLU is $\sigma'(x)=1$ for all $x>0$. Nonetheless, the derivative
remains zero for $x<0$, which does not seem to make it a natural choice for an activation function, and often leads to ``dead'' ReLUs. This problem has been partially addressed with ReLU variants, such as leaky ReLUs \cite{MHN13}, parameterized ReLUs \cite{HZRS15}, or maxout units \cite{GWMC+13}. These remain amongst the most popular choice of non-linearities as they allow for fast network training in practice.

\begin{figure}[t!]
  \begin{center}
    \includegraphics[width=0.95\textwidth]{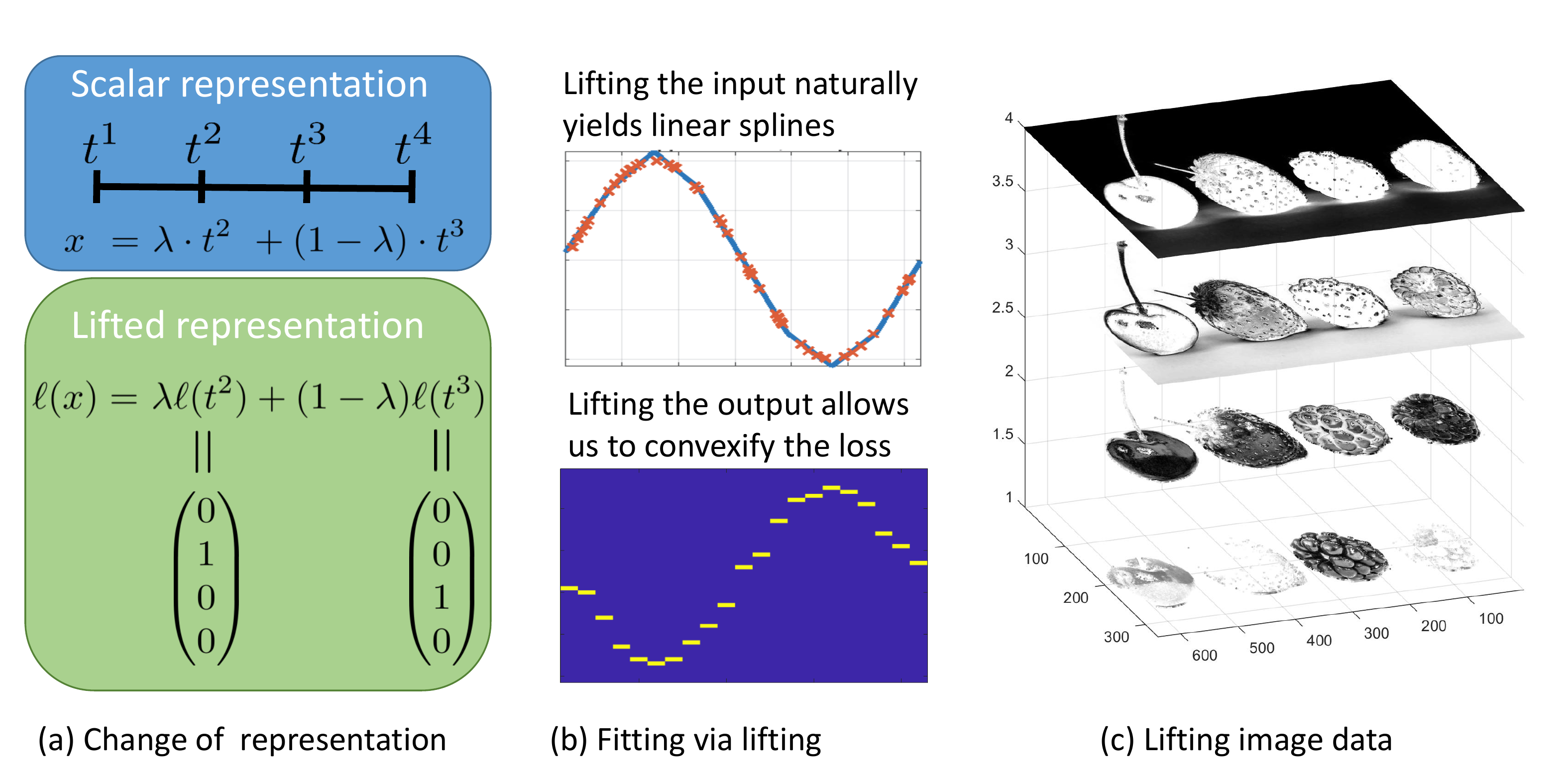}
  \end{center}
  \caption{\label{fig:teaser}The proposed lifting identifies predefined labels $t^i \in \mathbb{R}$ with the unit vectors $e_i$ in $\R^L$, $L\geq2$. As illustrated in (a), a number $x$ that is represented as a convex combination of $t^i$ and $t^{i+1}$ has a natural representation in a higher dimensional \textit{lifted} space, see \eqref{eq:ll}. When a lifting layer is combined with a fully connected layer it corresponds to a linear spline, and when both the input as well as the desired output are \textit{lifted} it allows non-convex cost functions to be represented as a convex minimization problem (b). Finally, as illustrated in (c), coordinate-wise lifting yields an interesting representation of images, which allows textures of different intensities to be filtered differently.}
\end{figure}

In this paper we propose a novel type of non-linear layer, which we call \textit{lifting layer} $\lift$. In contrast to ReLUs \eqref{eq:relu}, it does not discard large parts of the input data, but rather \textit{lifts} it to different channels that allow the input $x$ to be processed independently on different intervals. As we discuss in more detail in Section~\ref{sec:practicalScaledLifting}, the simplest form of the proposed lifting non-linearity is the mapping 
\begin{align}
\label{eq:reducedLifting}
 \sigma(x) = \begin{pmatrix}
 \max(x,0) \\ \min(x,0)
 \end{pmatrix},
 \end{align}
which 
essentially consists of two complementary ReLUs and therefore 
neither discards half of the incoming inputs nor has intervals of zero gradients. 

More generally, the proposed non-linearity depends on \textit{labels} $t^1< \hdots < t^L \in \mathbb{R}$ (typically linearly spaced) and is defined as a function $\map{\lift}{\mathbb{R}}{\mathbb{R}^L}$ that maps a scalar input $x \in \mathbb{R}$ to a vector $\lift(x)\in \mathbb{R}^L$ via
\begin{align}
\label{eq:ll}
\lift(x) = \Big(0, \hdots, 0, \underbrace{\frac{t^{l+1}-x}{t^{l+1}-t^l}}_{\text{$l$-th coordinate}} , \frac{x - t^l}{t^{l+1}-t^l}, 0, \hdots, 0\Big)^T ~  \text{ for } x \in [t^l, t^{l+1}].
\end{align}

The motivation of the proposed lifting non-linearity is illustrated in Figure~\ref{fig:teaser}. In particular, we highlight the following \emph{contributions}: 
\begin{enumerate}
\item The concept of representing a low dimensional variable in a higher dimensional space is a well-known optimization technique called \textit{functional lifting}, see \cite{PCBC10}. Non-convex problems are reformulated as the minimization of a convex energy in the higher dimensional 'lifted' space. While the {\bf introduction of lifting layers} does not directly correspond to the optimization technique, some of the advantageous properties carry over as we detail in Section~\ref{sec:liftingLayer}.  
\item ReLUs are commonly used in deep learning for imaging applications, however their low dimensional relatives of interpolation or regression problems are typically tackled differently, e.g. by fitting (piecewise) polynomials. We show that a lifting layer followed by a fully connected layer {\bf yields a linear spline}, which {\bf closes the gap between low and high dimensional interpolation problems}. In particular, the aforementioned architecture can {\bf approximate any continuous function} $\map{f}{\R}{\R}$ to arbitrary precision and can still be trained {\bf by solving a \textit{convex} optimization problem} whenever the loss function is convex, a favorable property that is, for example, not shared even by the simplest ReLU-based architecture.
\item By additionally lifting the desired output of the network, one can {\bf represent non-convex cost functions in a convex fashion}. Besides handling the non-convexity, such an approach allows for the minimization of cost functions with large areas of zero gradients such as truncated linear costs. 
\item We demonstrate that the proposed lifting {\bf improves the test accuracy in comparison to similar ReLU-based architectures in several experiments} on image classification and produces state-of-the-art image denoising results, making it an attractive universal tool in the design of neural networks.
\end{enumerate}

\section{Related Work}   \label{sec:relatedWork}

\paragraph{Lifting in Convex Optimization.} One motivation for the proposed non-linearity comes from a technique called \textit{functional lifting} which allows particular types of non-convex optimization problems to be reformulated as convex problems in a higher dimensional space, see \cite{PCBC10} for details. The recent advances in functional lifting \cite{MLMLC16} have shown that \eqref{eq:ll} is a particularly well-suited discretization of the continuous model from \cite{PCBC10}. Although, the techniques differ significantly, we hope for the general idea of an easier optimization in higher dimensions to carry over. Indeed, for simple instances of neural network architecture, we prove several favorable properties for our lifting layer that are related to properties of functional lifting. Details are provided in Sections~\ref{sec:liftingLayer} and~\ref{sec:outputLifting}.

\paragraph{Non-linearities in Neural Networks.} While many non-linear transfer functions have been studied in the literature (see \cite[Section 6.3]{GBC16} for an overview), the ReLU in \eqref{eq:relu} remains the most popular choice. Unfortunately, it has the drawback that its gradient is zero for all $x<0$, thus preventing gradient based optimization techniques to advance if the activation is zero (dead ReLU problem). Several variants of the ReLU avoid this problem by either utilizing smoother activations such as softplus \cite{DBBNG01} or exponential linear units \cite{CUH15}, or by considering
\begin{align}
\label{eq:generalizedRelu}
\sigma(x;\alpha) = \max(x,0) + \alpha \min(x,0),
\end{align}
e.g. the absolute value rectification $\alpha = -1$ \cite{JKRL09}, leaky ReLUs with a small $\alpha>0$ \cite{MHN13}, randomized leaky ReLUs with randomly choosen $\alpha$ \cite{XWCL15}, parametric ReLUs in which $\alpha$ is a learnable parameter \cite{HZRS15}. 
Self-normalizing neural networks \cite{KUMH17} use scaled exponential LUs (SELUs) which have further normalizing properties and therefore replace the use of batch normalization techniques \cite{IS15}.
While the activation \eqref{eq:generalizedRelu} seems closely related to the simplest case \eqref{eq:reducedLifting} of our lifting, the latter allows to process $\max(x,0)$ and $\min(x,0)$ separately, avoiding the problem of predefining $\alpha$ in \eqref{eq:generalizedRelu} and leading to more freedom in the resulting function. 

Another related non-linear transfer function are maxout units \cite{GWMC+13}, which (in the 1-D case we are currently considering) are defined as 
\begin{align}
\label{eq:maxout}
 \sigma(x) = \max_{j} (\theta_j x + b_j).
\end{align}
They can represent any piecewise linear \textit{convex} function. However, as we show in Proposition~\ref{prop:lifting-leads-to-lin-spline}, a combination of the proposed lifting layer with a fully connected layer drops the restriction to \textit{convex} activation functions, and allows us to learn \textit{any} piecewise linear function. 
This special architecture shows also similarities to learning the non-linear activation function in terms of basis functions \cite{CP17}.

\paragraph{Universal Approximation Theorem.}
As an extension of the universal approximation theorem in \cite{Cybenko1989}, it has been shown in \cite{LLPS93} that the set of feedforward networks with one hidden layer, i.e., all functions $\mathcal{N}$ of the form
\begin{align}
\label{eq:oneHiddenFF}
 \mathcal{N}(x)= \sum_{j=1}^N \theta^1_j \sigma(\langle \theta^2_j, x \rangle + b_j)
 \end{align}
for some integer $N$, and weights $\theta^1_j \in \mathbb{R}$, $\theta^2_j \in \mathbb{R}^n$, $b_j\in \mathbb{R}$ are dense in the set of continuous functions $\map{f}{[0,1]^n}{\R}$ if and only if $\sigma$ is not a polynomial. While this result demonstrates the expressive power of all common activation functions, the approximation of some given function $f$ with a network $\mathcal{N}$ of the form \eqref{eq:oneHiddenFF} requires optimization for the parameters $\theta^1$ and $(\theta^2,b)$ which inevitably leads to a non-convex problem.
We prove the same expressive power of a lifting based architecture (see Corollary~\ref{cor:universalApproximation}), while, remarkably, our corresponding learning problem is a convex optimization problem. 
Moreover, beyond the qualitative density result for \eqref{eq:oneHiddenFF}, we may quantify the approximation quality depending on a simple measure for the ``complexity'' of the continuous function to be approximated (see Corollary~\ref{cor:universalApproximation} and the Appendix~\ref{appdx:vector-val}).


\section{Lifting Layers}   \label{sec:liftingLayer}

In this section, we introduce the proposed lifting layers (Section~\ref{sec:lift-layer-def}) and study their favorable properties in a simple 1-D setting (Section~\ref{sec:lift-layer-1D-ana}). The restriction to 1-D functions is mainly for illustrative purposes and simplicity. All results can be transferred to higher dimensions via a vector-valued lifting (Section~\ref{sec:vec-lift-layer}). The analysis provided in this section does not directly apply to deep networks, however it provides an intuition for this setting. Section~\ref{sec:practicalScaledLifting} discusses some practical aspects and reveals a connection to ReLUs.
All proofs and the details of the vector-valued lifting are provided in Appendix~\ref{appdx:vector-val} and~\ref{appdx:lift-output}.

\subsection{Definition} \label{sec:lift-layer-def}

The following definition formalizes the lifting layer from the introduction.
\begin{definition}[Lifting] \label{def:lifting}
We define the lifting of a variable $x\in [\lI,\rI]$, $\lI,\rI\in\R$,  with respect to the Euclidean basis $\EE:=\set{e^1,\ldots,e^L}$ of $\R^L$ and a knot sequence $\lI=t^1< t^2 < \ldots <t^{L}=\rI$, for some $L\in \N$, as a mapping $\map{\lift}{[\lI,\rI]}{\R^L}$ given by
\begin{equation} \label{eq:def-lift}
  \lift(x) = (1-\lambda_l(x)) e^l + \lambda_{l}(x) e^{l+1}
  \quad  \text{with}\ l\ \text{such that}\ x\in [t^l, t^{l+1}]\,,
\end{equation}
where $\lambda_l(x) := \frac{x-t^l}{t^{l+1}-t^l}\in\R$. The inverse mapping $\map{\ilift}{\R^L}{\R}$ of $\lift$, which satisfies $\ilift(\lift(x))=x$, is defined by 
\begin{equation} \label{eq:def-inverse-lift}
  \ilift(z) = \sum_{l=1}^L z_l t^l \,.
\\\end{equation}
\end{definition}
Note that while liftings could be defined with respect to an arbitrary basis $\EE$ of $\R^L$ (with a slight modification of the inverse mapping), we decided to limit ourselves to the Euclidean basis for the sake of simplicity. Furthermore, we limit ourselves to inputs $x$ that lie in the predefined interval $[\lI,\rI]$. 
Although, the idea extends to the entire real line by linear extrapolation, it requires more technical details. For the sake of a clean presentation, we omit these details. 

\subsection{Analysis in 1D} \label{sec:lift-layer-1D-ana}

Although, here we are concerned with 1-D functions, these properties and examples provide some intuition for the implementation of the lifting layer into a deep architecture. Moreover, analogue results can be stated for the lifting of higher dimensional spaces.

\begin{proposition}[Prediction of a Linear Spline] \label{prop:lifting-leads-to-lin-spline}
  The composition of a fully connected layer $z\mapsto \scal{\theta}{z}$ with $\theta\in \R^L$, and a lifting layer, i.e., 
\begin{align}
\label{eq:linearSpline}
 \net_{\theta}(x) := \scal{\theta}{\lift(x)}, 
 \end{align}
yields a linear spline (continuous piecewise linear function). Conversely, any linear spline can be expressed in the form of \eqref{eq:linearSpline}. 
\end{proposition}
Although the architecture  in \eqref{eq:linearSpline} does not fall into the class of functions covered by the universal approximation theorem, well-known results of linear spline interpolation still guarantee the same results. 
\begin{corollary}[Prediction of Continuous Functions]
\label{cor:universalApproximation}
  Any continuous function $\map{f}{[\lI,\rI]}{\R}$ can be represented arbitrarily accurate with a network architecture $\net_\theta(x) := \scal{\theta}{\lift(x)}$ for sufficiently large $L$, $\theta\in \R^L$.
\end{corollary}
Furthermore, as linear splines can of course fit any (spatially distinct) data points exactly, our simple network architecture has the same property for a particular choice of labels $t^i$. On the other hand, this result suggests that using a small number of labels acts as regularization of the type of linear interpolation.
\begin{corollary}[Overfitting]  \label{prop:overfitting-1D}
  Let $(x_i,y_i)$ be training data, $i=1,\hdots,N$ with $x_i\neq x_j$ for $i\neq j$. If $L=N$ and $t^i = x_i$, there exists $\theta$ such that $\net_{\theta}(x):=\scal{\theta}{\lift(x)}$ is exact at all data points $x=x_i$, i.e. $\net_{\theta}(x_i) = y_i$ for all $i=1,\ldots,N$.
\end{corollary}
Note that Proposition~\ref{prop:lifting-leads-to-lin-spline} highlights two crucial differences of the proposed non-linearity to the maxout function in \eqref{eq:maxout}: \ii1 maxout functions can only represent convex piecewise linear functions, while liftings can represent arbitrary piecewise linear functions; \ii2 The maxout function is non-linear w.r.t. its parameters $(\theta_j, b_j)$, while the simple architecture in \eqref{eq:linearSpline} (with lifting) is linear w.r.t. its parameters $(\theta,b)$. The advantage of a lifting layer compared to a ReLU, which is less expressive and also non-linear w.r.t. its parameters, is even more significant.

Remarkably, the optimal approximation of a continuous function by a linear spline (for any choice of $t^i$), yields a convex minimization problem. 
\begin{proposition}[Convexity of a simple Regression Problem] \label{prop:convex-regression-prob}
  Let $(x_i,y_i) \in [\lI,\rI] \times \R$ be training data, $i=1,\hdots,N$. Then, the solution of the problem 
  \begin{align}
  \label{eq:linearSplineFitting}
    \min_{\theta} \sum_{i=1}^N\loss(\scal{\theta}{\lift(x_i)}; y_i)
  \end{align} 
  yields the best linear spline fit of the training data with respect to the loss function $\loss$. In particular, if $\loss$ is convex, then \eqref{eq:linearSplineFitting} is a convex optimization problem. 
\end{proposition}
As the following example shows, this is not true for ReLUs and maxout functions.
\begin{example} \label{ex:relu-non-convex}
 The convex loss $\mathcal{L}(z;1) = (z-1)^2$ composed with a ReLU applied to a linear transfer function, i.e., $\theta \mapsto \max(\theta x_i,0)$ with $\theta\in \R$, leads to a non-convex objective function, e.g. for $x_i=1$, $\theta\mapsto (\max(\theta,0)-1)^2$ is non-convex.
\end{example}

Therefore, in the light of Proposition~\ref{prop:convex-regression-prob}, the proposed lifting closes the gap between low dimensional approximation and regression problems (where linear splines are extremely common), and high dimensional approximation/learning problems, where ReLUs have been used instead of linear spline type of functions. 

\subsection{Vector-Valued Lifting Layers} \label{sec:vec-lift-layer}
A vector-valued construction of the lifting similar to \cite{LMMLC16} allows us to naturally extend all our previous results for functions $\map{f}{[\lI,\rI]}{\R}$ to functions $\map{f}{\Omega \subset \R^d}{\R}$. Definition~\ref{def:lifting} is generalized to $d$ dimensions by triangulating the compact domain $\Omega$, and identifying each vertex of the resulting mesh with a unit vector in a space $\R^N$, where $N$ is the total number of vertices. The lifted vector contains the barycentric coordinates of a point $x\in \R^d$ with respect its surrounding vertices. The resulting lifting remains a continuous piecewise linear function when combined with a fully connected layer (cf. Proposition~\ref{prop:lifting-leads-to-lin-spline}), and yields a convex problem when looking for the best piecewise linear fit on a given triangular mesh (cf. Proposition~\ref{prop:convex-regression-prob}). Intuition is provided in Figure~\ref{fig:vector-val-lift} and the details are provided in Appendix~\ref{appdx:vector-val}. 
\begin{figure}[t]
\begin{center}
  \begin{tikzpicture}
    \draw[->] (-0.5,0) -- (5,0);
    \draw[->] (0,-0.5) -- (0,4.5);

    \coordinate (V1) at (0.5,0.5);
    \coordinate (V2) at (1.5,0.5);
    \coordinate (V3) at (3,0.5);
    \coordinate (V4) at (1,1.5);
    \coordinate (V5) at (2.5,2);
    \coordinate (V6) at (4,1.5);
    \coordinate (V7) at (0.5,3);
    \coordinate (V8) at (1.5,3);
    \coordinate (V9) at (3,3.5);
    \coordinate (V10) at (4.5,3);
    \coordinate (V11) at (1.5,4);
    \coordinate (V12) at (4,4.5);
    
    \fill[color=orange!5!white] (V1) -- (V2) -- (V3) -- (V6) -- (V10) -- (V12) -- (V11) -- (V7) -- (V4) -- (V1);
    \draw[orange] (1,-0.3) node[right] {$\Omega=\bigcup_{l=1}^{13} T^{l}$};

    \filldraw[opacity=0.6] (V1) circle [radius=1pt] node[below]      {$V^1$};
    \filldraw[opacity=0.6] (V2) circle [radius=1pt] node[below]      {$V^2$};
    \filldraw[opacity=0.6] (V3) circle [radius=1pt] node[below]      {$V^3$};
    \filldraw[opacity=0.6] (V4) circle [radius=1pt] node[left]       {$V^4$};
    \filldraw[opacity=0.6] (V5) circle [radius=1pt] node[below]      {$V^5$};
    \filldraw[opacity=0.6] (V6) circle [radius=1pt] node[below right]{$V^6$};
    \filldraw[opacity=0.6] (V7) circle [radius=1pt] node[left]       {$V^7$};
    \filldraw[opacity=0.6] (V8) circle [radius=1pt] node[below left] {$V^8$};
    \filldraw[opacity=0.6] (V9) circle [radius=1pt] node[above]      {$V^9$};
    \filldraw[opacity=0.6] (V10) circle [radius=1pt] node[right]     {$V^{10}$};
    \filldraw[opacity=0.6] (V11) circle [radius=1pt] node[above]     {$V^{11}$};
    \filldraw[opacity=0.6] (V12) circle [radius=1pt] node[right]     {$V^{12}$};

    \draw (V1) -- (V2) -- (V4) -- cycle;
    \draw (V2) -- (V3) -- (V5) -- cycle;
    \draw (V3) -- (V5) -- (V6) -- cycle;
    \draw (V4) -- (V5) -- (V8) -- cycle;
    \draw (V4) -- (V7) -- (V8) -- cycle;
    \draw (V7) -- (V8) -- (V11) -- cycle;
    \draw (V8) -- (V9) -- (V11) -- cycle;
    \draw (V2) -- (V3) -- (V5) -- cycle;
    \draw (V5) -- (V8) -- (V9) -- cycle;
    \draw (V5) -- (V9) -- (V10) -- cycle;
    \draw (V5) -- (V6) -- (V10) -- cycle;
    \draw (V9) -- (V11) -- (V12) -- cycle;
    \draw (V9) -- (V10) -- (V12) -- cycle;

    \coordinate (T1) at (barycentric cs:V1=0.33,V2=0.33,V4=0.33); 
    \coordinate (T2) at (barycentric cs:V2=0.33,V4=0.33,V5=0.33); 
    \coordinate (T3) at (barycentric cs:V2=0.33,V3=0.33,V5=0.33); 
    \coordinate (T4) at (barycentric cs:V3=0.33,V5=0.33,V6=0.33); 
    \coordinate (T5) at (barycentric cs:V4=0.33,V7=0.33,V8=0.33); 
    \coordinate (T6) at (barycentric cs:V4=0.33,V5=0.33,V8=0.33); 
    \coordinate (T7) at (barycentric cs:V5=0.33,V8=0.33,V9=0.33); 
    \coordinate (T8) at (barycentric cs:V5=0.33,V9=0.33,V10=0.33); 
    \coordinate (T9) at (barycentric cs:V5=0.33,V6=0.33,V10=0.33); 
    \coordinate (T10) at (barycentric cs:V7=0.33,V8=0.33,V11=0.33); 
    \coordinate (T11) at (barycentric cs:V8=0.33,V9=0.33,V11=0.33); 
    \coordinate (T12) at (barycentric cs:V9=0.33,V11=0.33,V12=0.33); 
    \coordinate (T13) at (barycentric cs:V9=0.33,V10=0.33,V12=0.33); 

    \node[color=orange,opacity=0.6] at (T1) {$T^{1}$};
    \node[color=orange,opacity=0.6] at (T2) {$T^{2}$};
    \node[color=orange,opacity=0.6] at (T3) {$T^{3}$};
    \node[color=orange,opacity=0.6] at (T4) {$T^{4}$};
    \node[color=orange,opacity=0.6] at (T5) {$T^{5}$};
    \node[color=orange,opacity=0.6] at (T6) {$T^{6}$};
    \node[color=orange,opacity=0.6] at (T7) {$T^{7}$};
    \node[color=orange,opacity=0.6] at (T8) {$T^{8}$};
    \node[color=orange,opacity=0.6] at (T9) {$T^{9}$};
    \node[color=orange,opacity=0.6] at (T10) {$T^{10}$};
    \node[color=orange,opacity=0.6] at (T11) {$T^{11}$};
    \node[color=orange,opacity=0.6] at (T12) {$T^{12}$};
    \node[color=orange,opacity=0.6] at (T13) {$T^{13}$};

    \begin{scope}[shift={(7,0)}]
      \node (lV0) at (0,0) {};
      \foreach \x/\y in {0/1,1/2,2/3,3/4,4/5,5/6,6/7,7/8,8/9,9/10,10/11,11/12}
      {
        \node[above=0mm of lV\x.north,draw,minimum size=4mm] (lV\y) {};
      }
      \node[below] (lV0) {\color{red}$\ell(x)\in \R^{12}$};
    \end{scope}
    
    \coordinate (X) at (barycentric cs:V5=0.2,V8=0.5,V9=0.3);
    
    \draw[dashed,color=blue!70!cyan,thick] (V5) -- (X);
    \draw[dashed,color=green!70!black,thick] (V8) -- (X);
    \draw[dashed,color=red!70!cyan,thick] (V9) -- (X);

    \draw[->,blue!70!cyan,thick] (X) to[bend left=15] (lV5);
    \draw[->,green!70!black,thick] (X) to[bend left=15] (lV8);
    \draw[->,red!70!cyan,thick] (X) to[bend left=15] (lV9);

    \foreach \x in {1,2,3,4,6,7,10,11,12}
    {
      \draw (lV\x) node {\scriptsize 0};
    }
    \draw (lV5) node {\scriptsize 0.2};
    \draw (lV8) node {\scriptsize 0.5};
    \draw (lV9) node {\scriptsize 0.3};
    
    \filldraw[color=red,thick] (X) circle [radius=2pt] node[below left] {$x$};

  \end{tikzpicture}
\end{center}
\caption{\label{fig:vector-val-lift}Intuition and notation of the vector-valued lifting.}
\end{figure}
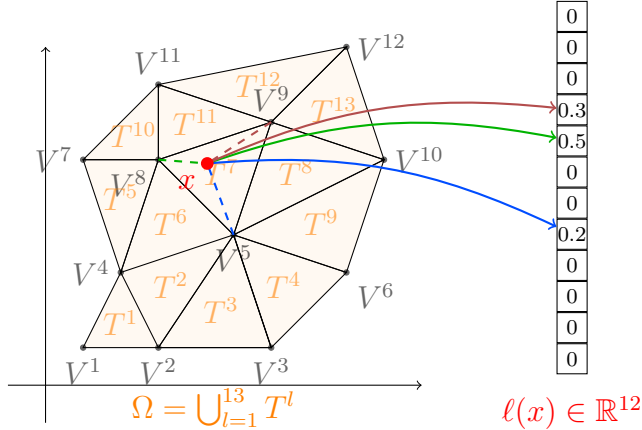
Unfortunately, discretizing a domain $\Omega \subset \R^d$ with $L$ labels per dimension leads to $N=L^d$ vertices, which makes a vector-valued lifting prohibitively expensive for large $d$. Therefore, in high dimensional applications, we turn to narrower and deeper network architectures, in which the scalar-valued lifting is applied to each component separately. The latter sacrifices the convexity of the overall problem for the sake of a high expressiveness with comparably few parameters. Intuitively, the increasing expressiveness is explained by an exponentially growing number of kinks for the composition of layers that represent linear splines. A similar reasoning can be found in \cite{MPCB14}.

\subsection{Scaled Lifting}
\label{sec:practicalScaledLifting}

We are free to scale the lifted representation defined in \eqref{eq:def-lift}, when the inversion formula in \eqref{eq:def-inverse-lift} compensates for this scaling. For practical purposes, we found it to be advantageous to also introduce a \textbf{scaled lifting} by replacing \eqref{eq:def-lift} in Definition~\ref{def:lifting} by
\begin{equation} \label{eq:def-lift-scaled}
  \scaledlift(x) = (1-\lambda_l(x)) t^l e^l + \lambda_{l}(x) t^{l+1} e^{l+1}
  \quad  \text{with}\ l\ \text{such that}\ x\in [t^l, t^{l+1}]\,,
\end{equation}
where $\lambda_l(x) := \frac{x-t^l}{t^{l+1}-t^l}\in\R$. The inversion formula reduces to the sum over all components of the vector in this case. We believe that such a scaled lifting is often advantageous: \ii1 The magnitude/meaning of the components of the lifted vector is preserved and does not have to be learned; \ii2 For an uneven number of equally distributed labels in $[-\rI,\rI]$, one of the labels $t^l$ will be zero, which allows us to omit it and represent a scaled lifting into $\R^L$ with $L-1$ many entries. For $L=3$ for example, we find that $t^1 =-\rI$, $t^2 = 0$, and $t^3=\rI$ such that 
\begin{equation} 
  \scaledlift(x) = \left\{\begin{alignedat}{5}
   \Big(1-&\frac{x+\rI}{0+\rI}\Big) (-&&\rI)&& e^1 &&= x e^1 && \text{ if } x\leq 0, \\
  &\frac{x-0}{\rI-0} &&\rI &&e^{3} &&= x e^3 && \text{ if } x >0.
  \end{alignedat}\right.
\end{equation}
As the second component remains zero, we can introduce an equivalent more memory efficient variant of the scaled lifting which we already stated in \eqref{eq:reducedLifting}.

\section{Lifting the Output} \label{sec:outputLifting}

So far, we considered liftings as a non-linear layer in a neural network. However, motivated by lifting-based optimization techniques, which seek a tight convex approximation to problems involving non-convex loss functions, this section presents a convexification of non-convex loss functions by lifting in the context of neural networks. This goal is achieved by approximating the loss by a linear spline and predicting the output of the network in a lifted representation. The advantages of this approach are demonstrated at the end of this section in Example~\ref{ex:lift-output-robust-fit} for a robust regression problem with a vast number of outliers.\\ 

Consider a loss function $\map{\loss_y}{\R}{\R}$ defined for a certain given output $y$ (the total loss for samples $(x_i,y_i)$, $i=1,\ldots,N$, may be given by $\sum_{i=1}^N \loss_{y_i}(x_i)$). We achieve the tight convex approximation by a lifting function $\map{\lift_y}{[\lI_y, \rI_y]}{\R^{L_y}}$ for the range of the loss function $\im(\loss_y)\subset\R$ with respect to the standard basis $\EE_y=\{e_y^1,\ldots,e_y^{L_y}\}$ and a knot sequence $\lI_y = t_y^1 < \ldots < t_y^{L_y} < \rI_y$ following Definition~\ref{def:lifting}. 

The goal of the convex approximation is to predict the lifted representation of the loss, i.e. a vector $z\in \R^{L_y}$. However, in order to assign the correct loss to the lifted variable, it needs to lie in $\im(\lift_y)$. In this case, we have a one-to-one representation of the loss between $[\lI_y, \rI_y]$ and $\im(\lift_y)$, which is shown by the following lemma.
\begin{lemma}[{Characterization of the Range of $\lift$}]
\label{lem:char-range-lift}
  The range of the lifting $\map{\lift}{[\lI,\rI]}{\R^L}$ is given by
  \begin{equation} \label{eq:def-range-lift}
    \im(\lift) = \set{z\in [0,1]^L \setsep 
                  \exists l\colon z_l+z_{l+1} = 1 \ \text{and}\
                  \forall k\not\in \{ l,l+1\}\colon z_k=0
                  } 
  \end{equation}
  and the mapping $\lift$ is a bijection between $[\lI,\rI]$ and $\im(\lift)$ with inverse $\ilift$.
\end{lemma}
Since the image of the range of $\lift_y$ is not convex, we relax it to a convex set, actually to the smallest convex set that contains $\im(\lift_y)$, the convex hull of $\im(\lift_y)$.
\begin{lemma}[{Convex Hull of the Range of $\lift$}] \label{lem:conv-relax}
  The convex hull $\conv(\im(\lift))$ of $\im(\lift)$ is the unit simplex in $\R^L$.
\end{lemma}
\begin{figure}[t]
  \begin{center}
    \begin{tabular}{cccc}
    \includegraphics[width = 0.23\textwidth]{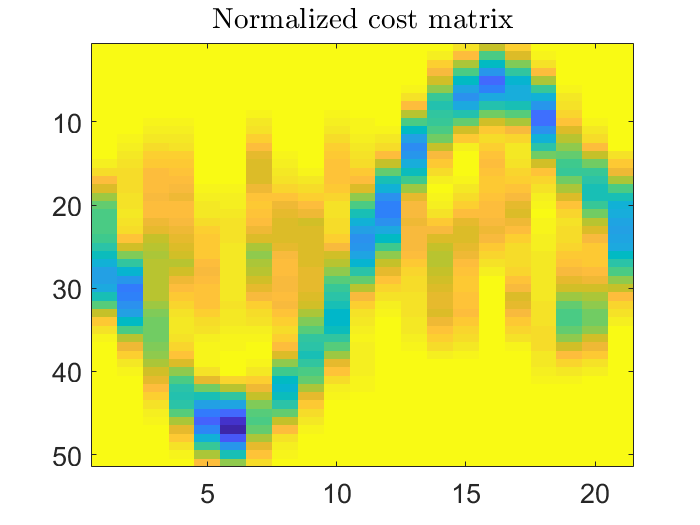}&
    \includegraphics[width = 0.23\textwidth]{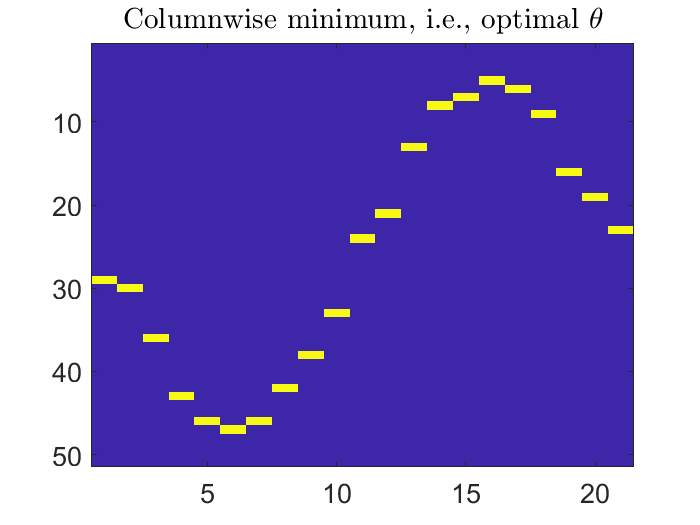}&
    \includegraphics[width = 0.23\textwidth]{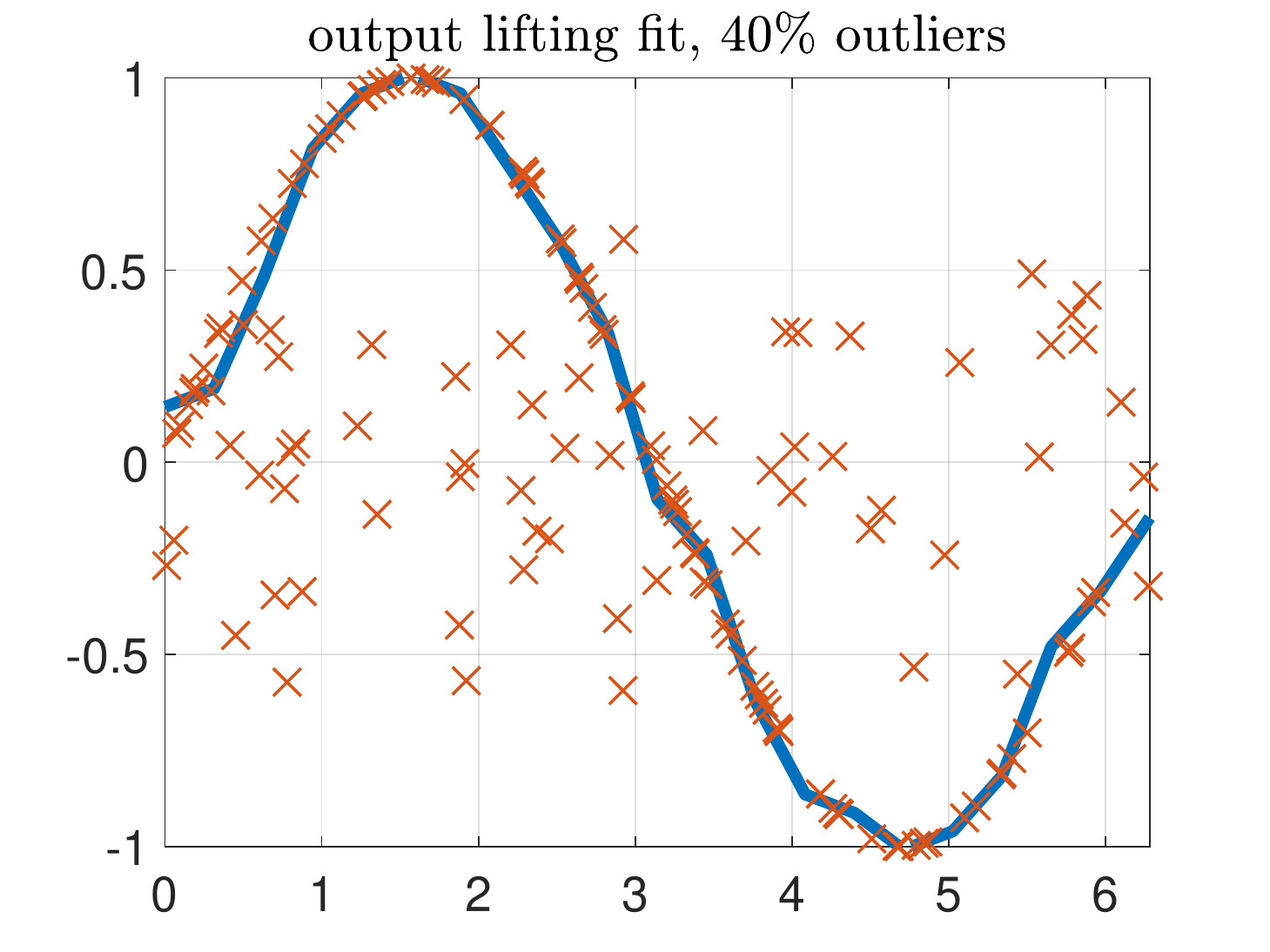}&
    \includegraphics[width = 0.23\textwidth]{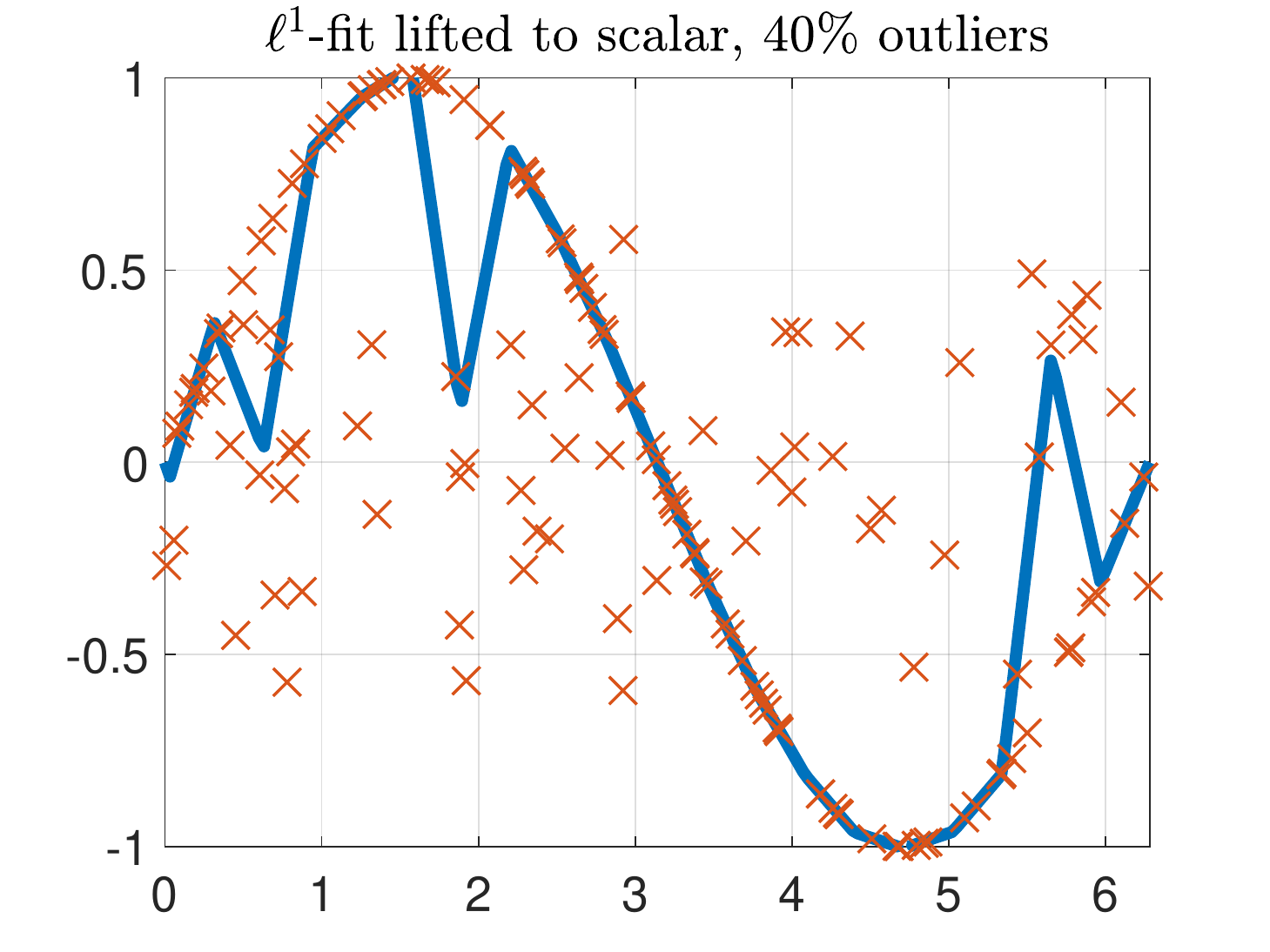}\\
    (a) Cost matrix $c$ & (b) Optimal $\theta$ & (c) Resulting fit & (d) Best $\ell^1$ fit \\
    \includegraphics[width = 0.23\textwidth]{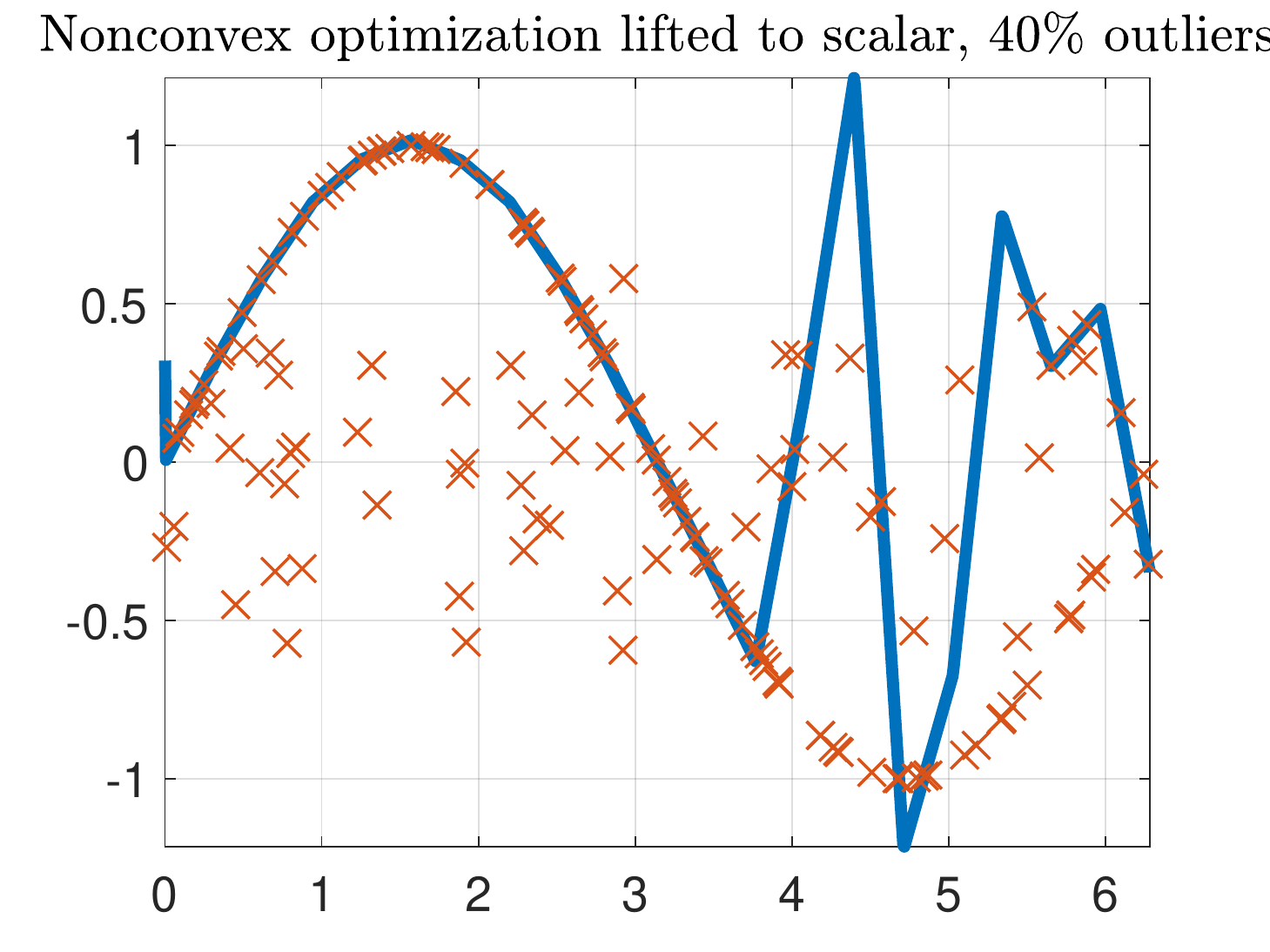}&
    \includegraphics[width = 0.23\textwidth]{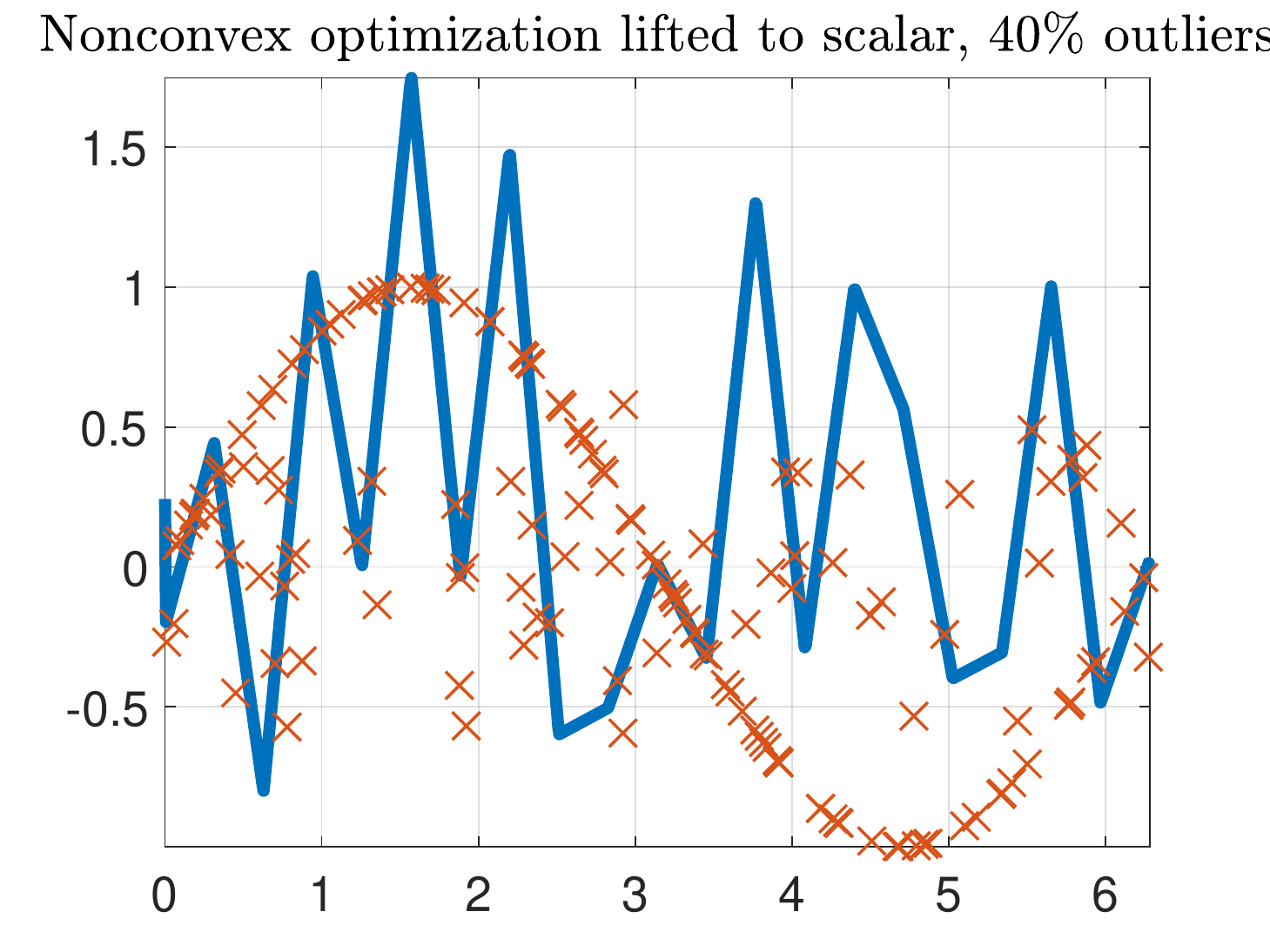}&
    \includegraphics[width = 0.23\textwidth]{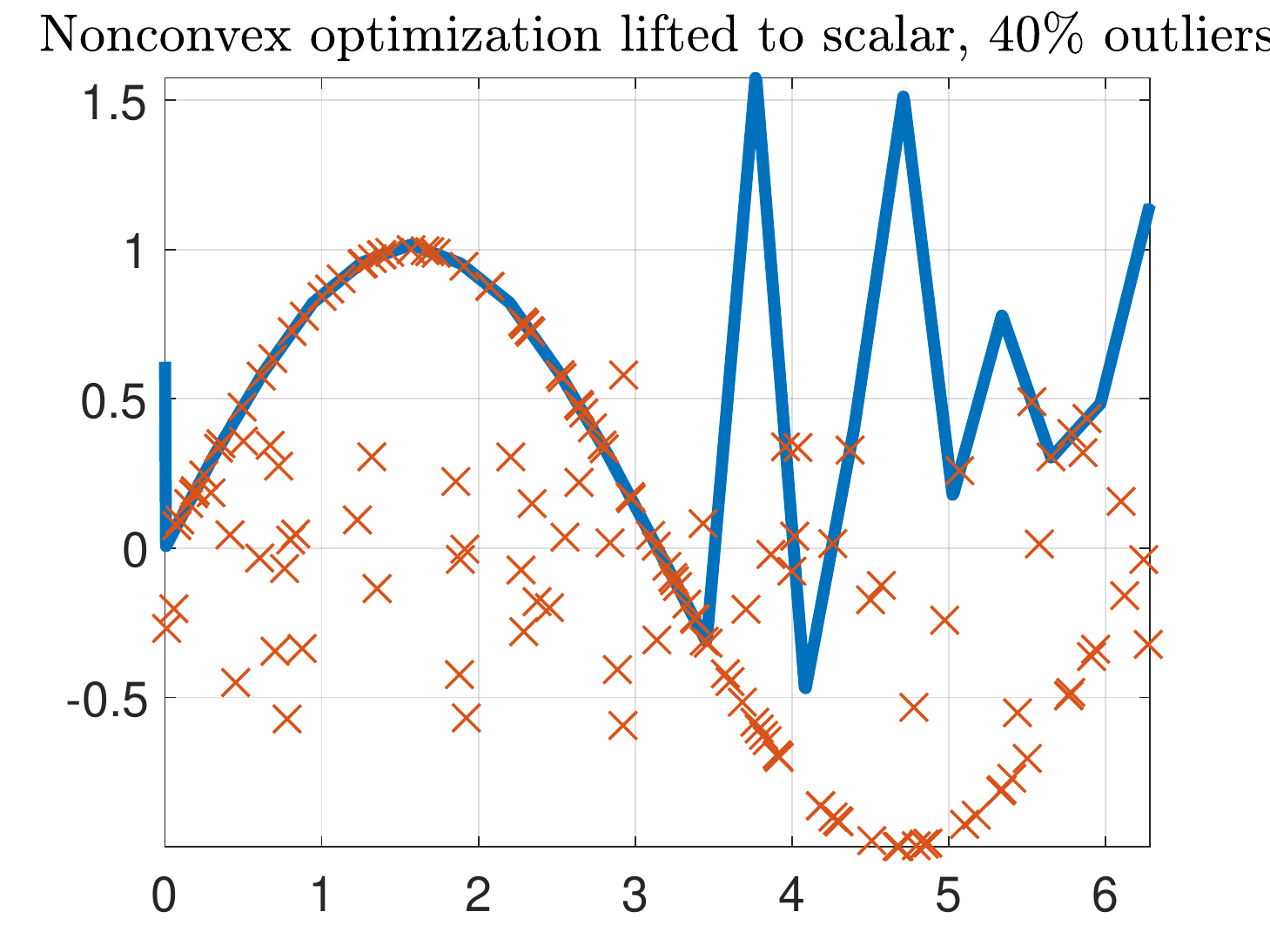}&
    \includegraphics[width = 0.23\textwidth]{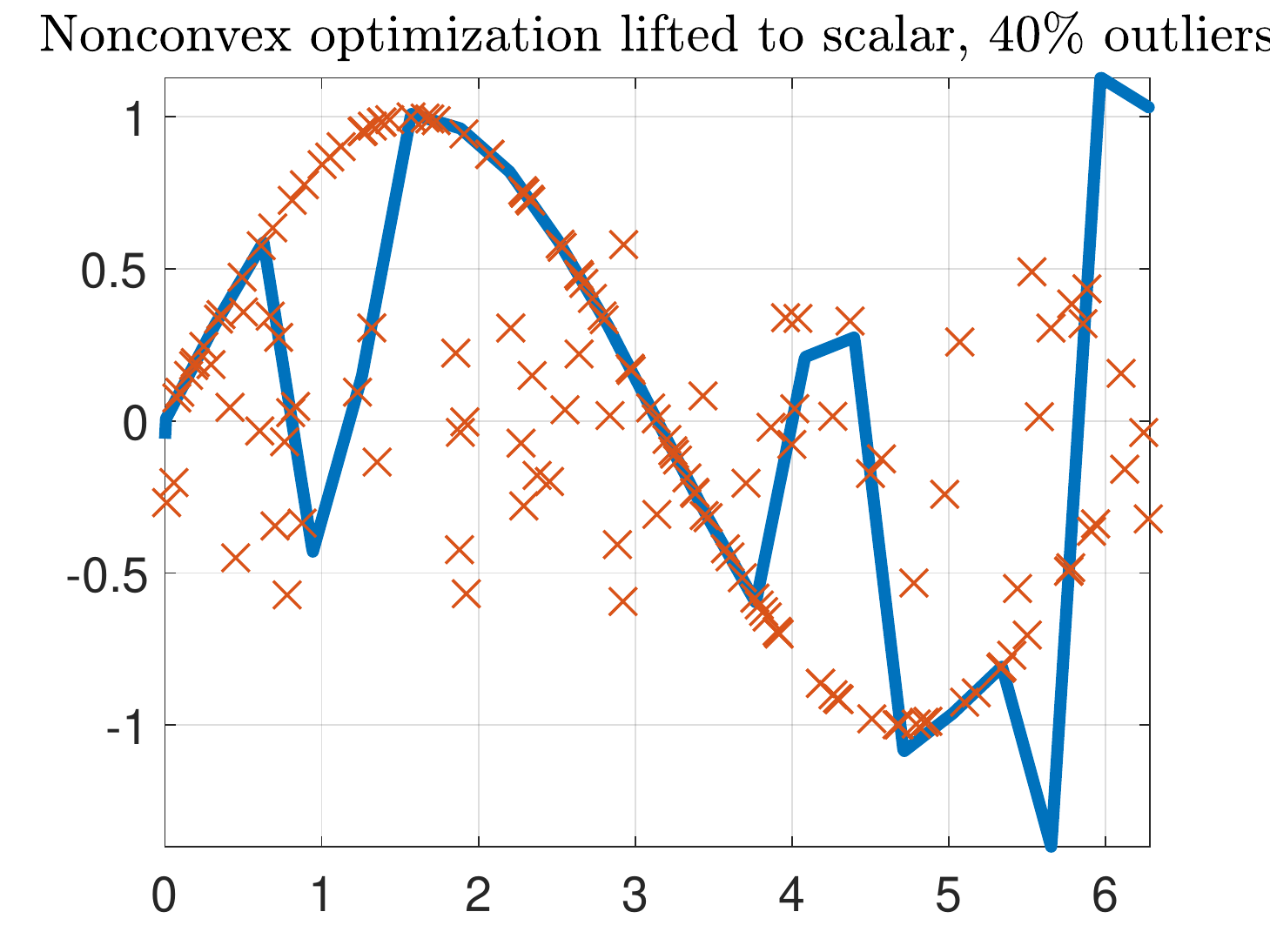}\\
    (e) Non-convex fit 1 & (f) Non-convex fit 2 & (g) Non-convex fit 3 & (h) Non-convex fit 4
    \end{tabular}
  \end{center}
  \caption{\label{fig:liftingTheOutput}Visualization of Example~\ref{ex:lift-output-robust-fit} for a regression problem with 40\% outliers. Our lifting of a (non-convex) truncated linear loss to a convex optimization problem robustly fits the function nearly optimally (see (c)), whereas the most robust convex formulation (without lifting) is severely perturbed by the outliers (see (d)). Trying to optimize the non-convex cost function directly yields different results based on the initialization of the weights and is prone to getting stuck in suboptimal local minima, see (e)-(h).}
\end{figure} 
Putting the results together, we obtain a tight convex approximation of the (possibly non-convex) loss function $\loss_y(x)$ by $\ilift_y(z)$ with $z\in \im(\lift_y)$, i.e. instead of considering a network $\net_\theta(x)$ and evaluate $\loss_y(\net_\theta(x))$, we consider a network $\widetilde\net_\theta(x)$ that predicts a point in $\conv(\im(\lift_y))\subset\R^{L_y}$ and evaluate the loss $\ilift_y(\widetilde\net_\theta(x))$. As it is hard to incorporate range-constraints into the network's prediction, we compose the network with a lifting layer $\lift_x$, i.e. we consider $\ilift_y(\tilde\theta\lift_x(\widetilde\net_\theta(x)))$ with $\tilde\theta\in \R^{L_y\times L_x}$, for which simpler constraints may be derived that can be handled easily. The following proposition states the convexity of the relaxed problem w.r.t. the parameters of the loss layer $\tilde \theta$ for a non-convex loss function $\loss_y$.

\begin{proposition}[Convex Relaxation of a simple non-convex Regression Problem] \label{prop:lift-output}
  Let $(x_i,y_i)\in [\lI,\rI]\times [\lI_y,\rI_y]$ be training data, $i=1,\ldots,N$. Moreover, let $\lift_y$ be a lifting of the common image $[\lI_y,\rI_y]$ of the loss functions $\loss_{y_i}$, $i=1,\ldots,N$, and $\lift_x$ is the lifting of the domain of $\loss_{y}$. Then
  \begin{equation} \label{eq:conv-rel-of-non-convex}
    \min_{\theta } \, \sum_{i=1}^N \ilift_y(\theta \lift_x(x_i)) \quad\st\  \theta_{p,q} \geq 0,\, \sum_{p=1}^{L_y} \theta_{p,q} = 1,\, 
    \begin{cases}
      \forall p=1,\ldots,L_y\,, \\ 
      \forall q=1,\ldots,L_x \,. 
    \end{cases}
  \end{equation}
is a convex relaxation of the (non-convex) loss function, and the constraints guarantee that $\theta \lift_x(x_i)\in \conv(\im(\lift_y))$.

  The objective in \eqref{eq:conv-rel-of-non-convex} is linear (w.r.t. $\theta$) and can be written as
  \begin{equation} \label{eq:lifted-output-cost-matrix}
      \sum_{i=1}^N \ilift_y(\theta \lift_x(x_i)) 
      = \sum_{i=1}^N \sum_{p=1}^{L_y} \sum_{q=1}^{L_x} \theta_{p,q} \lift_x(x_i)_q t^p_y 
      =:  \sum_{p=1}^{L_y} \sum_{q=1}^{L_x}  c_{p,q}  \theta_{p,q} 
  \end{equation}
  where $c:=\sum_{i=1}^N t_y \lift_x(x_i)^\top$, with $t_y := (t^1_y, \ldots, t^{L_y}_y)^\top$, is the cost matrix for assigning the loss value $t_y^p$ to the inputs $x_i$.

  Moreover, the closed-form solution of \eqref{eq:conv-rel-of-non-convex} is given for all $q=1,\ldots, L_x$ by $\theta_{p,q} = 1$, if the index $p$ minimizes $c_{p,q}$, and $\theta_{p,q}=0$ otherwise.
\end{proposition}

\begin{example}[Robust fitting] \label{ex:lift-output-robust-fit}
For illustrative purposes of the advantages of this section, we consider a regression problem with 40\% outliers as visualized in Figure~\ref{fig:liftingTheOutput}(c) and~(d). Statistics motivates us to use a robust non-convex loss function.
Our lifting allows us to use a robust (non-convex) truncated linear loss in a convex optimization problem (Proposition~\ref{prop:lift-output}), which can easily ignore the outliers and achieve a nearly optimal fit (see Figure~\ref{fig:liftingTheOutput}(c)), whereas the most robust convex loss (without lifting), the $\ell_1$-loss, yields a solution that is severely perturbed by the outliers (see Figure~\ref{fig:liftingTheOutput}(d)). The cost matrix $c$ from \eqref{eq:lifted-output-cost-matrix} that represents the non-convex loss (of this example) is shown in Figure~\ref{fig:liftingTheOutput}(a) and the computed optimal $\theta$ is visualized in Figure~\ref{fig:liftingTheOutput}(b). For comparison purposes we also show the results of a direct (gradient descent + momentum) optimization of the truncated linear costs with four different initial weights chosen from a zero mean Gaussian distribution. As we can see the results greatly differ for different initializations and always got stuck in suboptimal local minima. 
\end{example}

\section{Numerical Experiments}
\label{sec:Results}
In this section we provide synthetic numerical experiments to illustrate the behavior of lifting layers on simple examples, before moving to real-world imaging applications. We implemented lifting layers in MATLAB as well as in PyTorch and will make all code for reproducing the experiments available upon acceptance of this manuscript.

\subsection{Synthetic Examples}
The following results were obtained using a stochastic gradient descent (SGD) algorithm with a momentum of 0.9, using minibatches of size 128, and a learning rate of $0.1$. Furthermore, we use weight decay
 with a parameter of $10^{-4}$.

\subsubsection{1-D Fitting}

To illustrate our results of Proposition~\ref{prop:convex-regression-prob}, we first consider the example of fitting values $y_i = \sin(x_i)$ from input data $x_i$ sampled uniformly in $[0,2\pi]$. We compare the lifting-based architecture $\net_{\theta}(x) = \langle \theta, \lift(x)\rangle$ (Lift-Net) with the standard design architecture $\text{fc}_1(\sigma(\text{fc}_9(x)))$ (Std-Net), where $\sigma(x) = \max(x,0)$ applies coordinate-wise and $\text{fc}_n$ denotes a fully connected layer with $n$ output neurons. Figure~\ref{fig:1dExample} shows the resulting functions after 25, 75, 200, and 2000 epochs of training.

\begin{figure}[h]
  \begin{center}
    \includegraphics[width=0.24\textwidth]{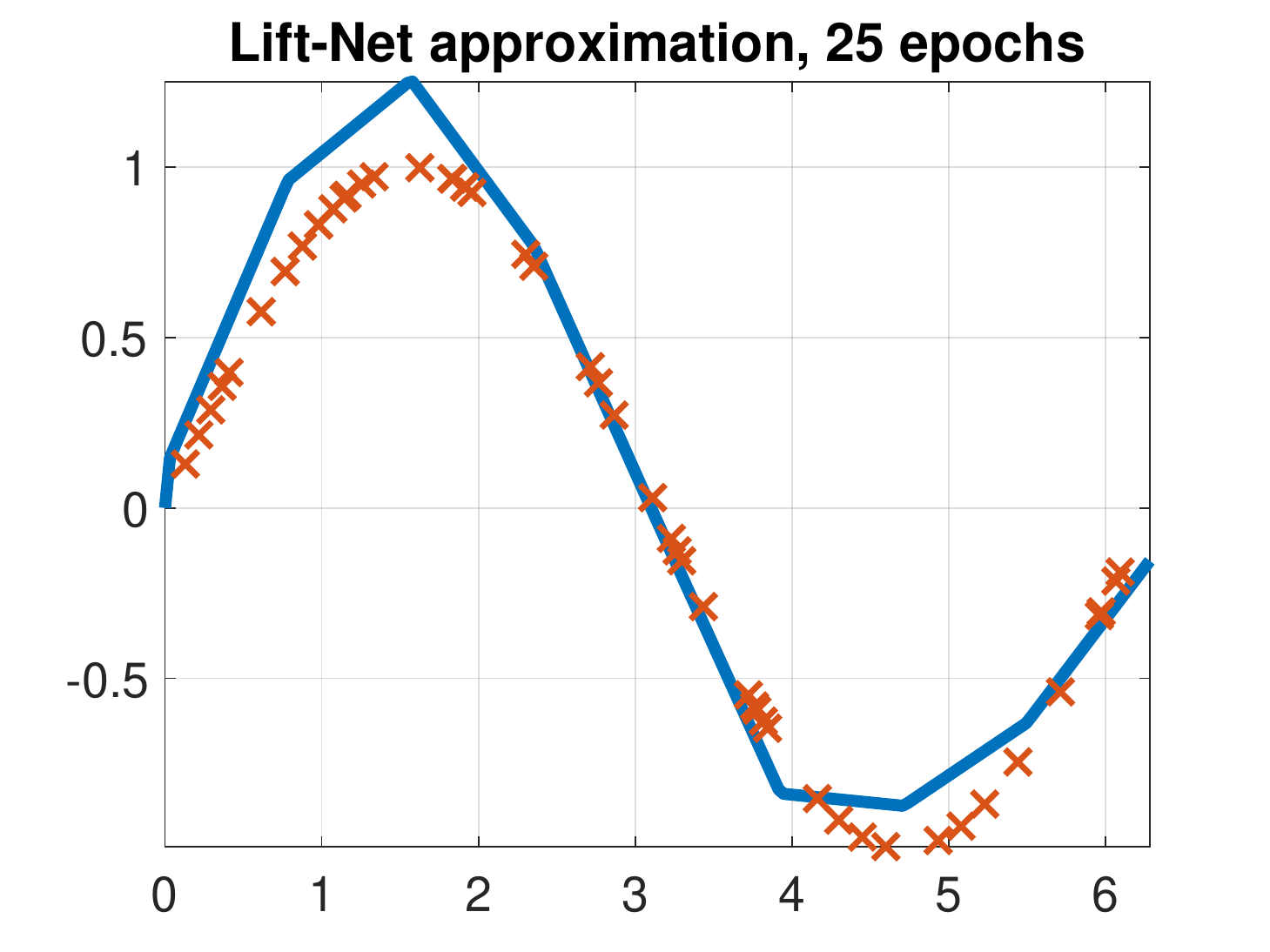}
    \includegraphics[width=0.24\textwidth]{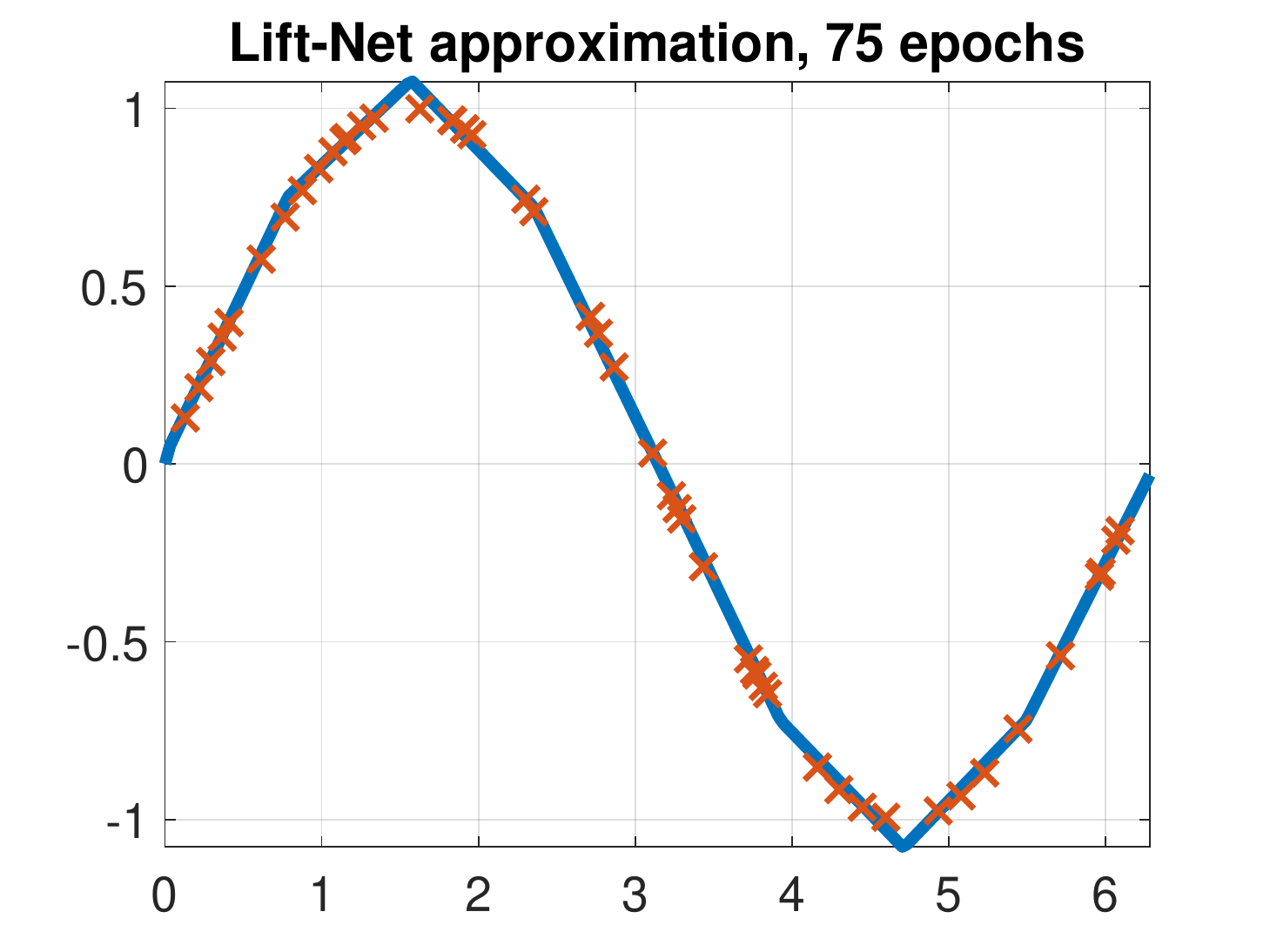}
    \includegraphics[width=0.24\textwidth]{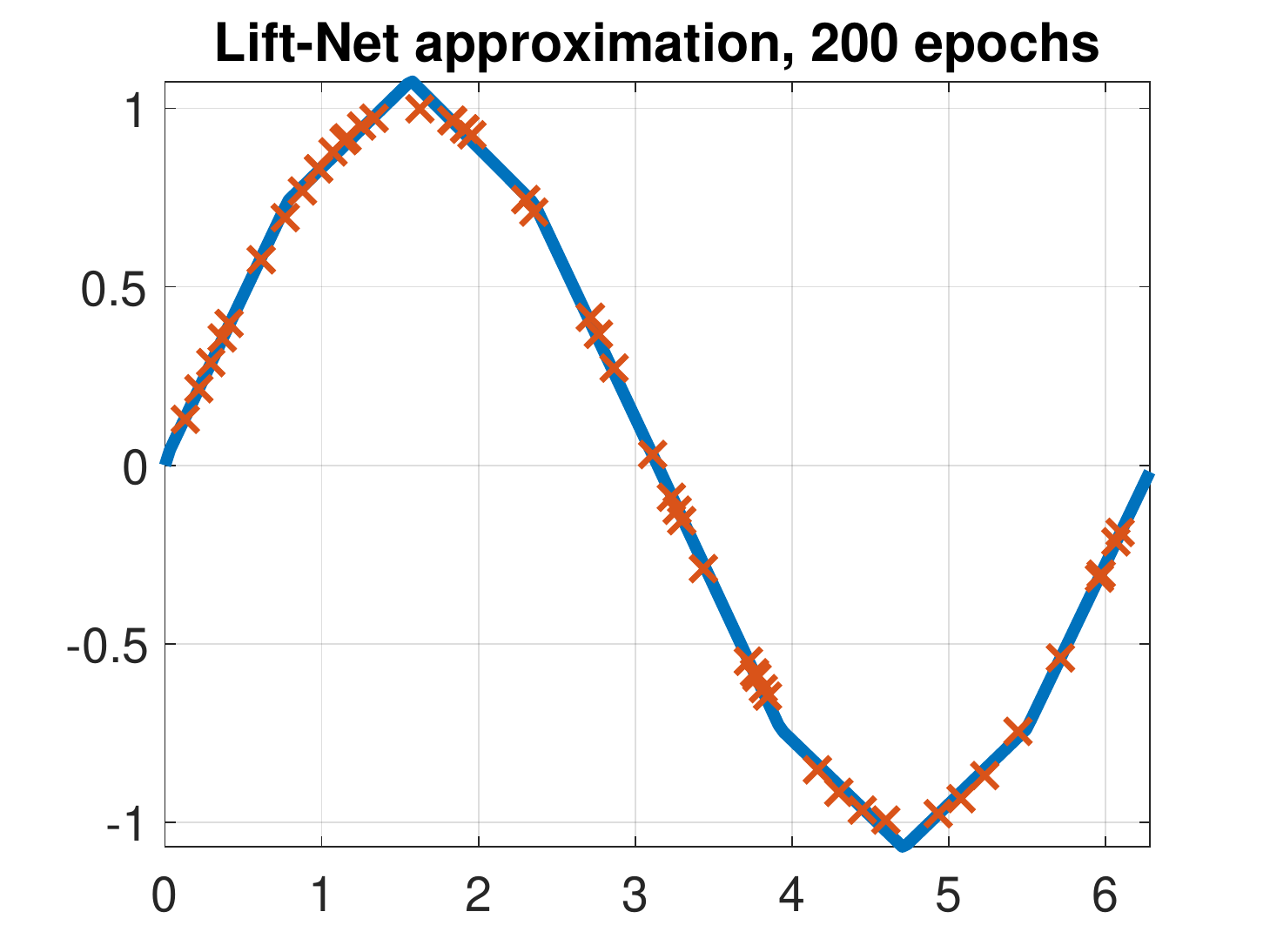}
    \includegraphics[width=0.24\textwidth]{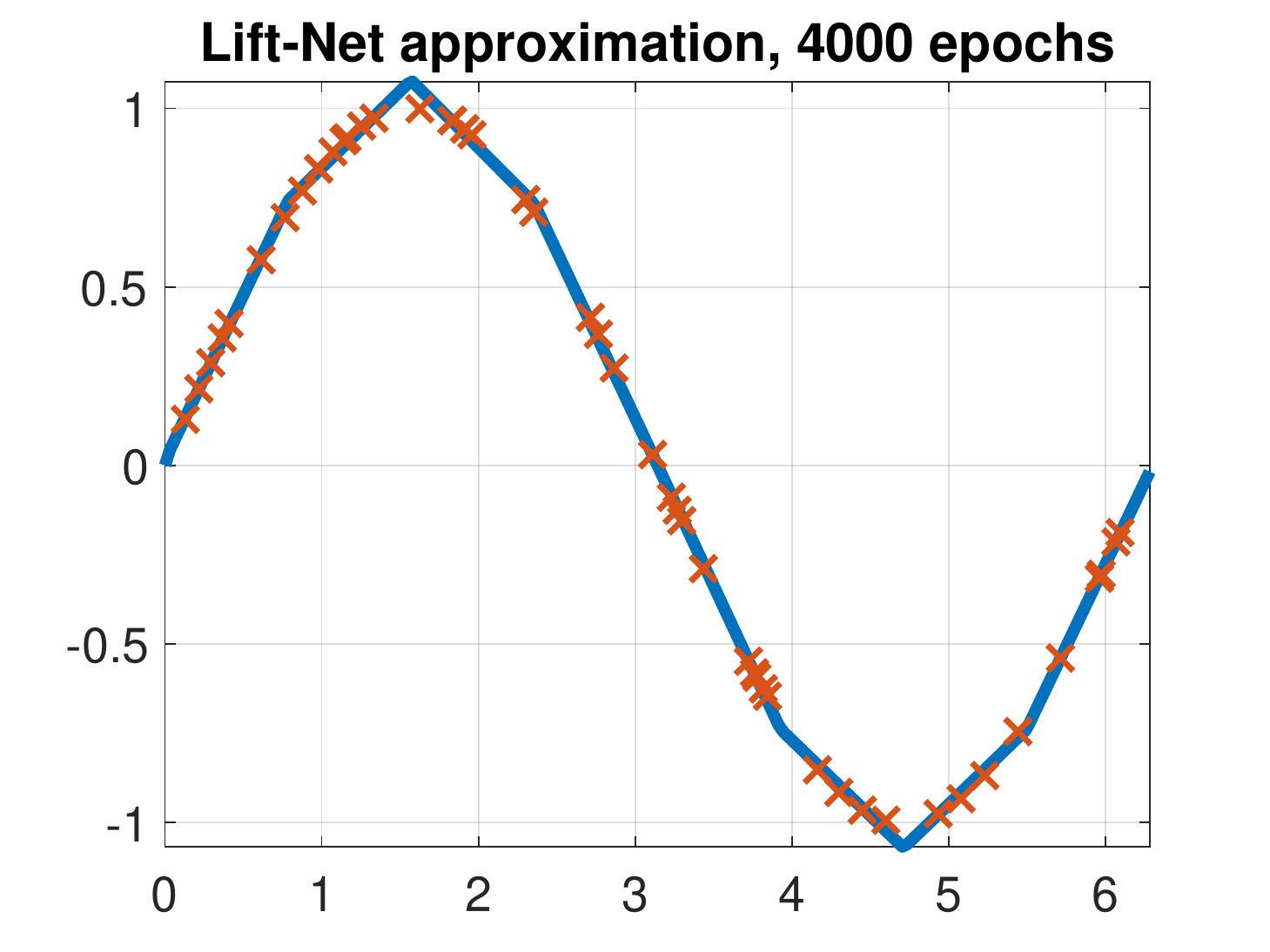}\\
    \includegraphics[width=0.24\textwidth]{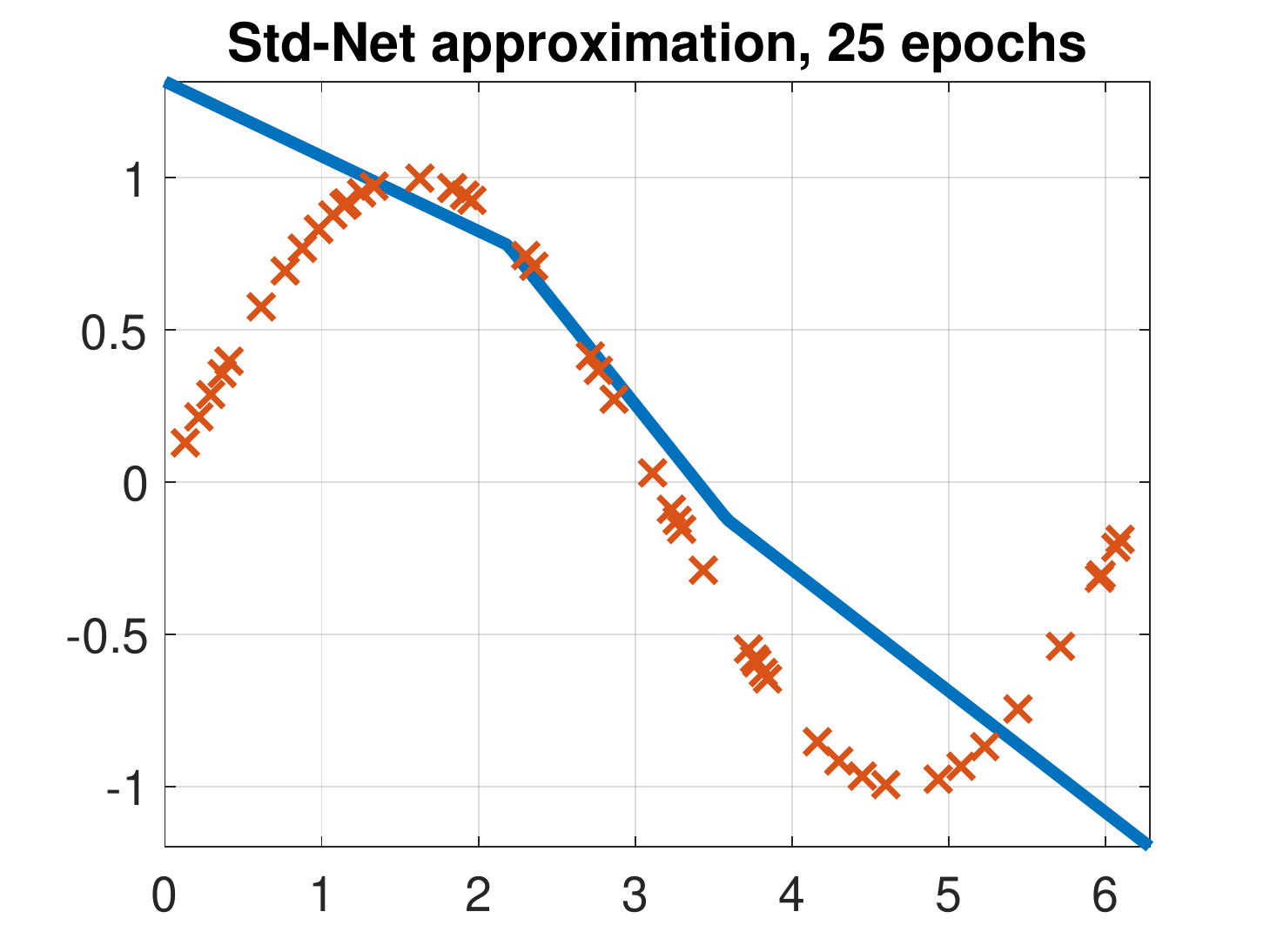}
    \includegraphics[width=0.24\textwidth]{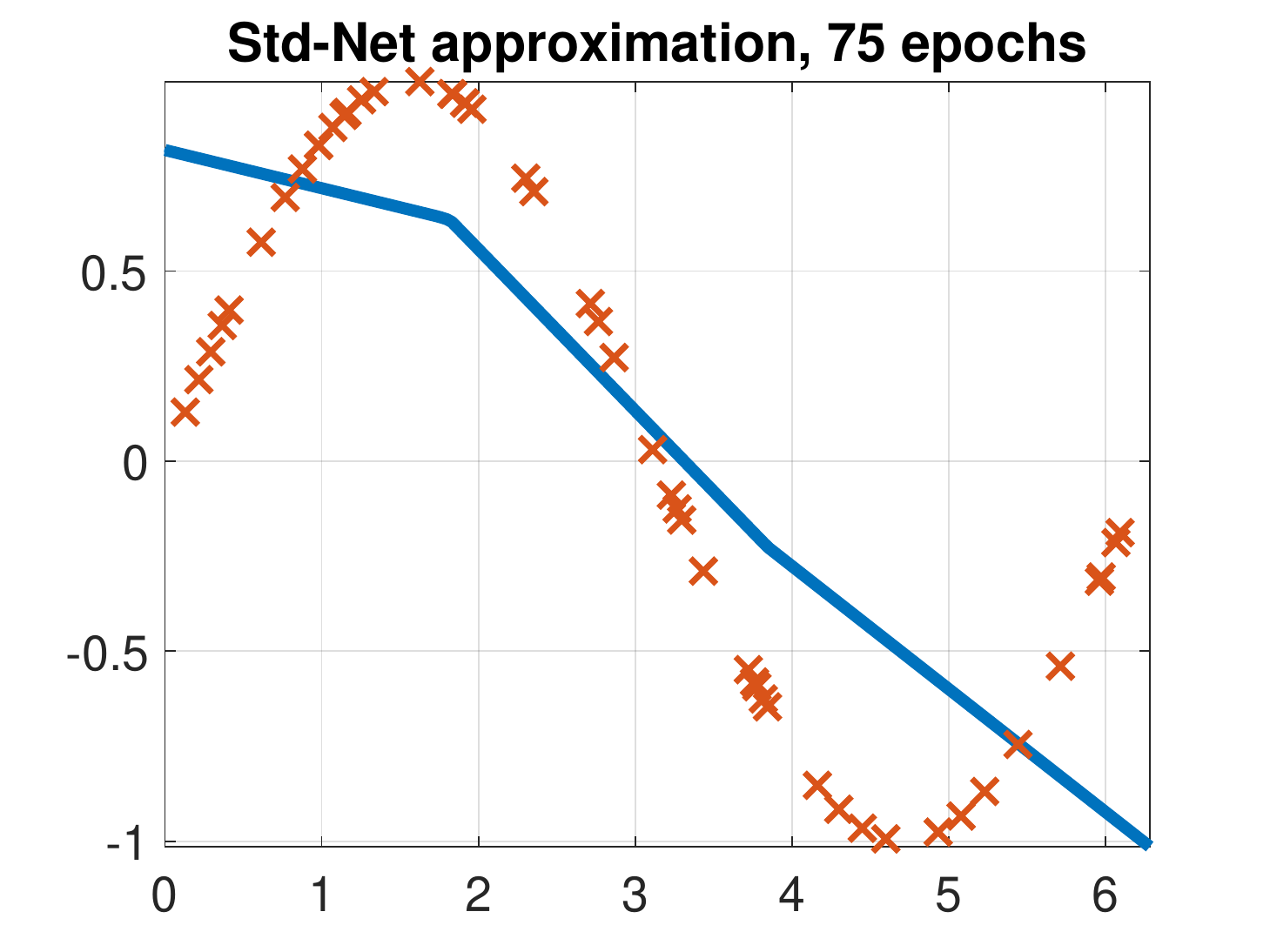}
    \includegraphics[width=0.24\textwidth]{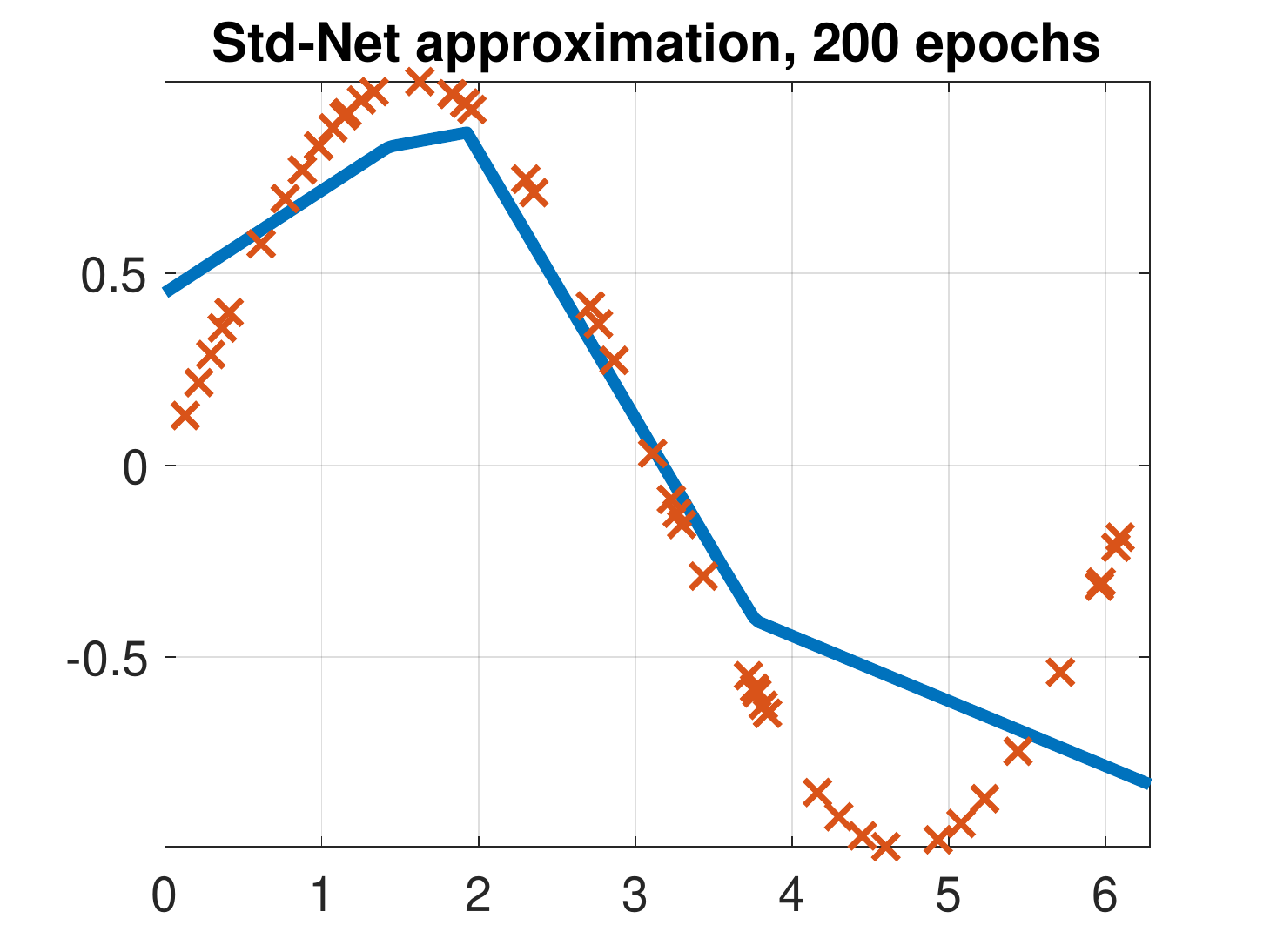}
    \includegraphics[width=0.24\textwidth]{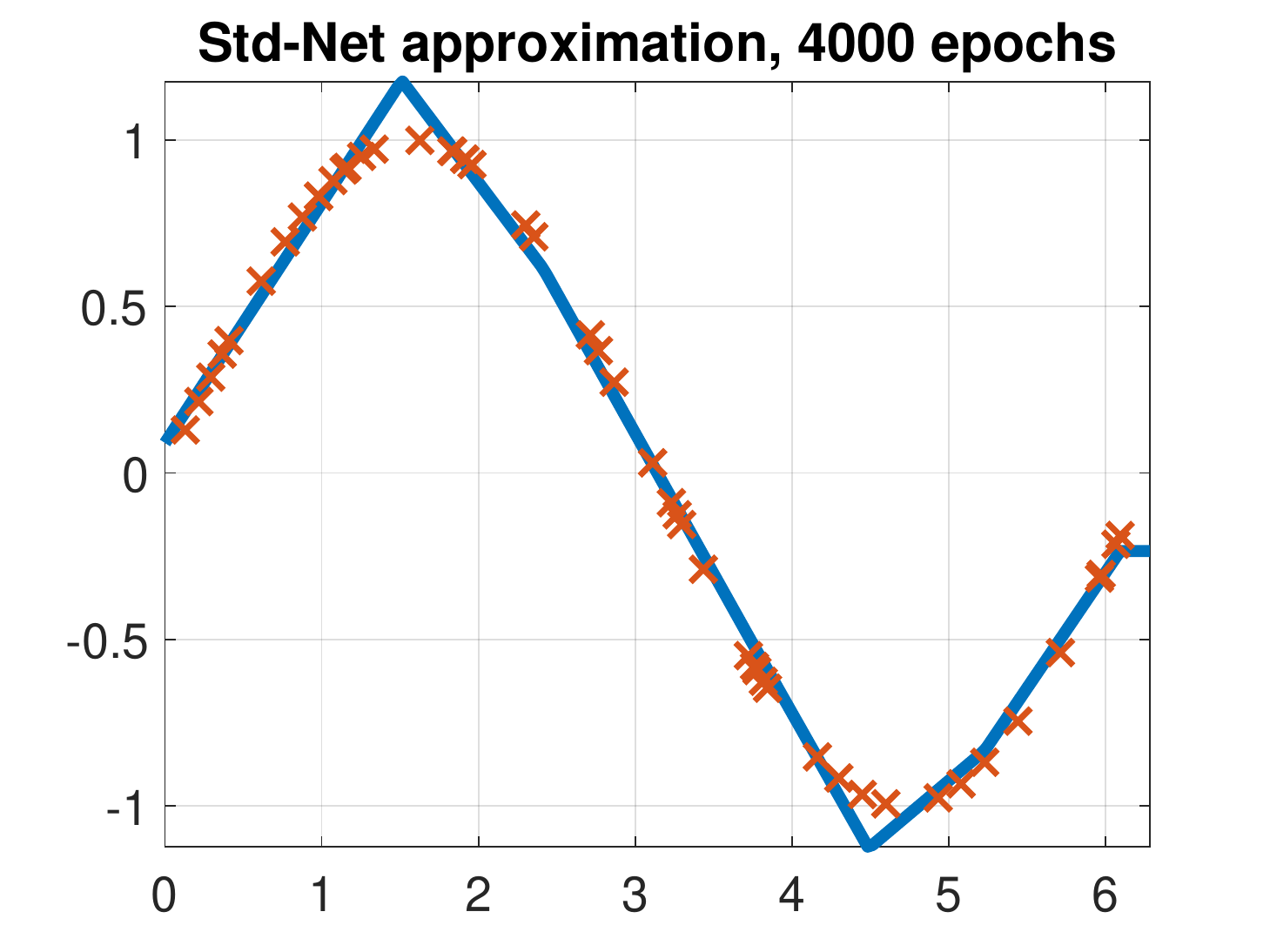}
  \end{center}
  \caption{\label{fig:1dExample}Illustrating the results of approximating a sine function on $[0,2\pi]$ with 50 training examples after different number of epochs. While the proposed architecture with lifting yields a convex problem for which SGD converges quickly (upper row), the standard architecture based on ReLUs yields an (ambiguous) non-convex problem which leads to slower convergence and a suboptimal local minimum after 4000 epochs (lower row). }
\end{figure}

\subsubsection{2-D Fitting}

\begin{figure}[h!]
  \begin{center}
    \includegraphics[width=0.3\textwidth]{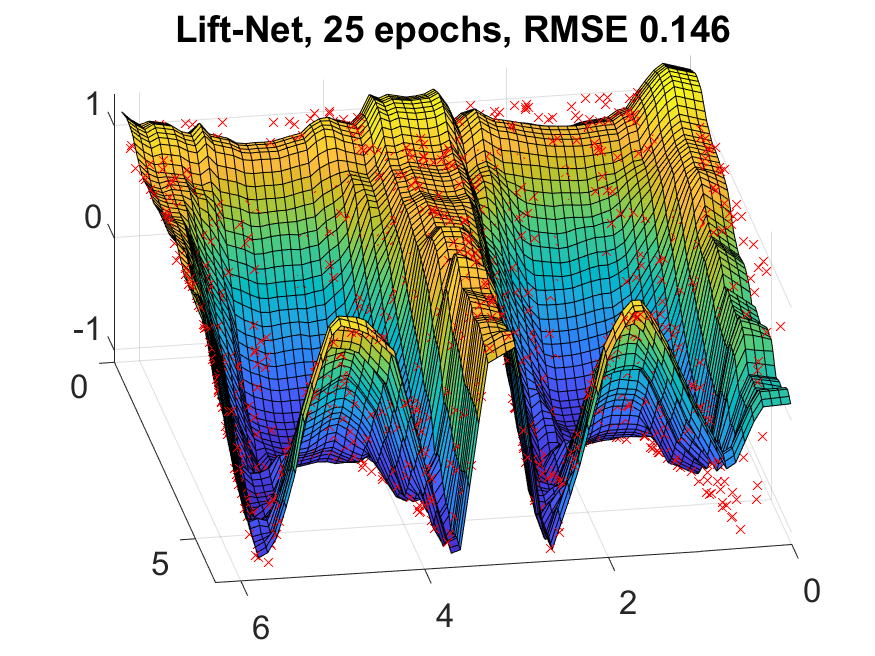}
    \includegraphics[width=0.3\textwidth]{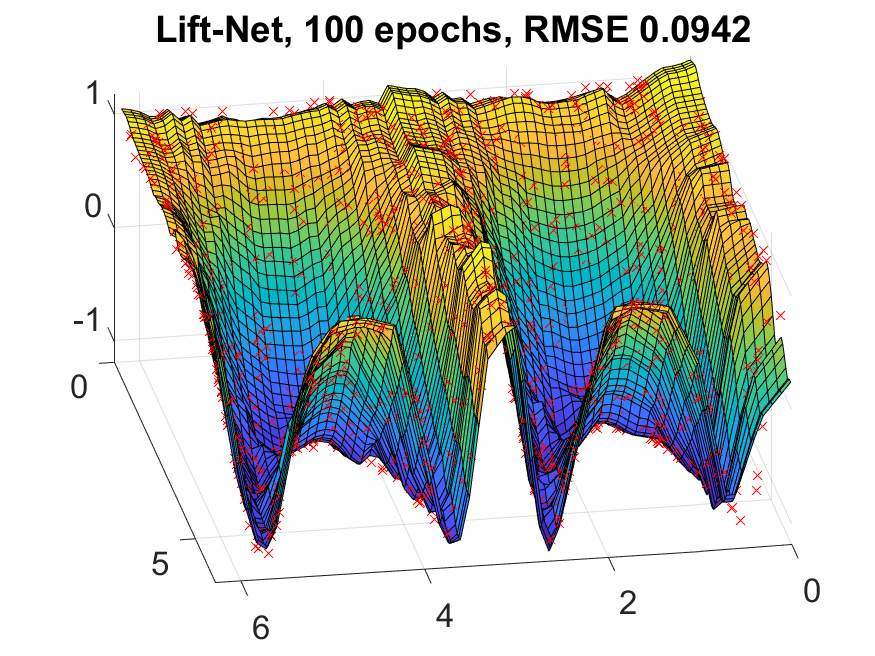}
    \includegraphics[width=0.3\textwidth]{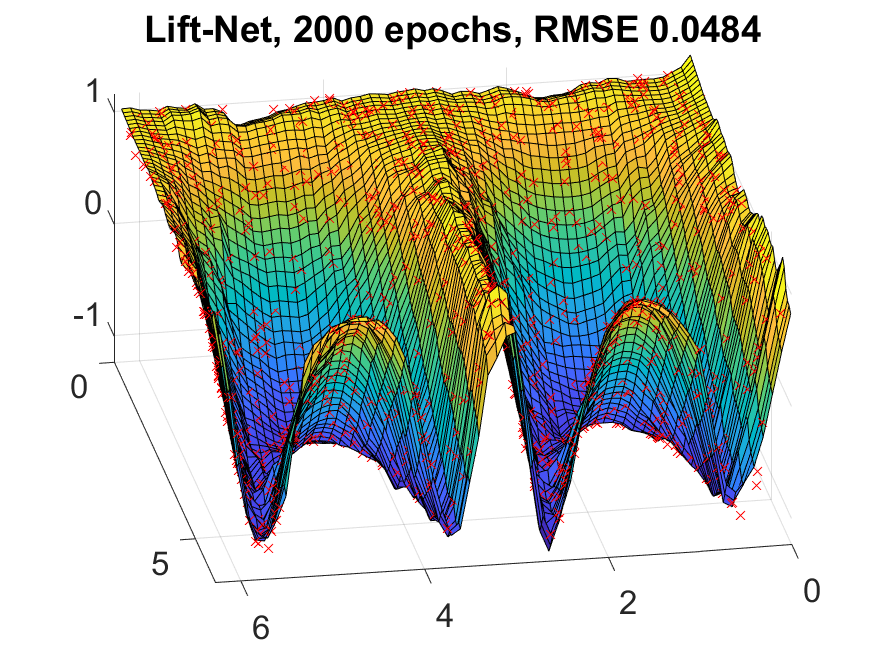}\\
    \includegraphics[width=0.3\textwidth]{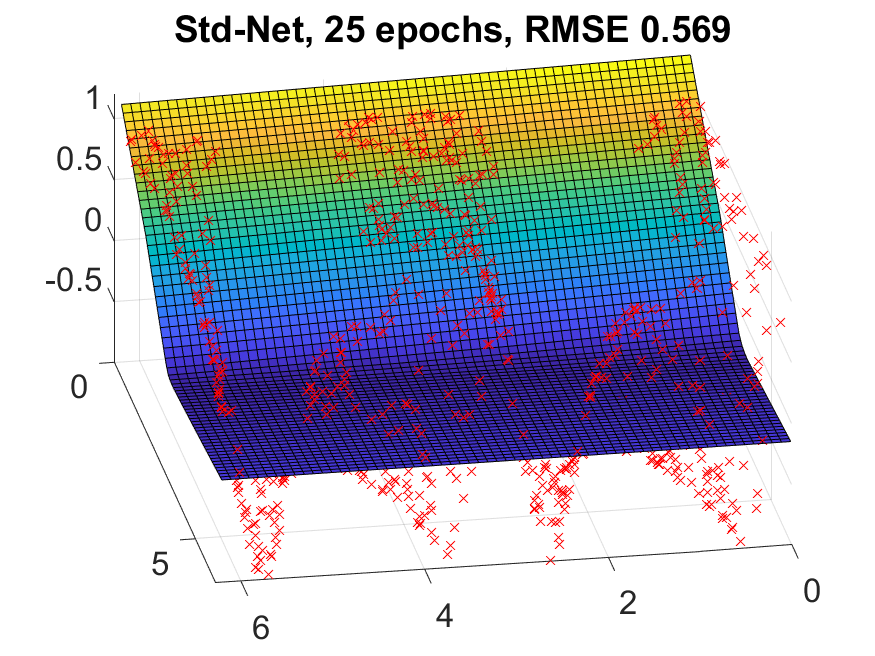}
    \includegraphics[width=0.3\textwidth]{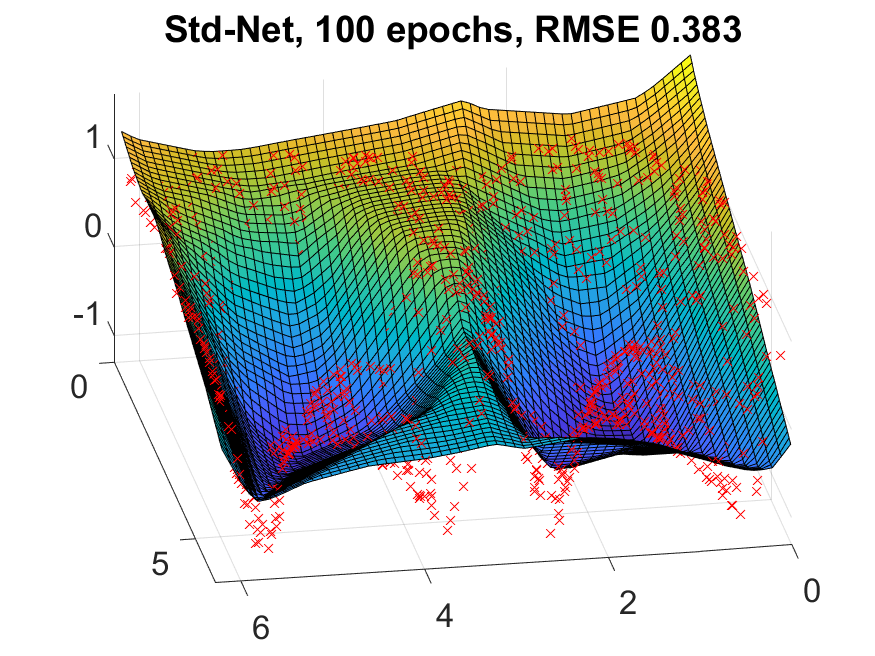}
    \includegraphics[width=0.3\textwidth]{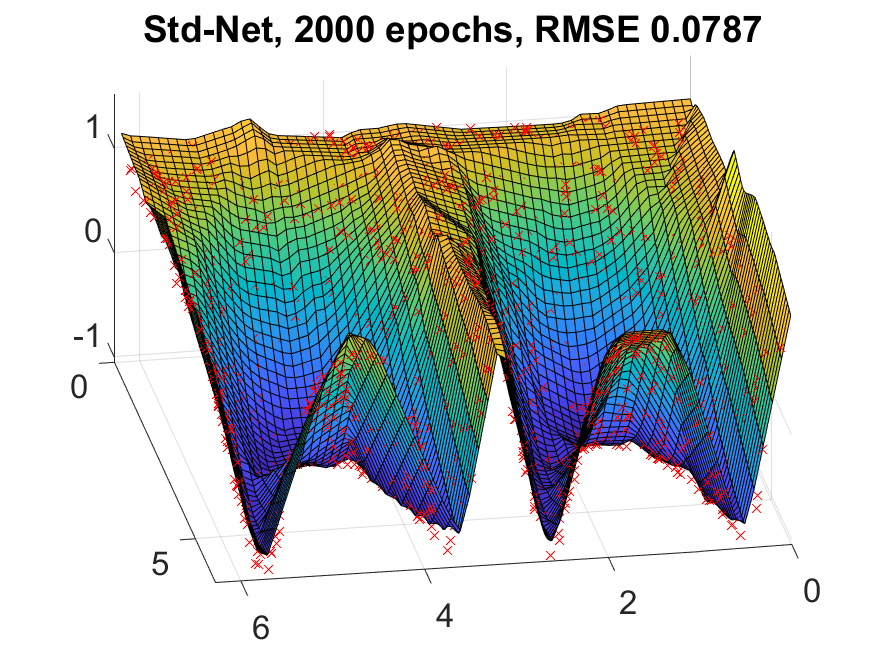}\\
    \includegraphics[width=0.3\textwidth]{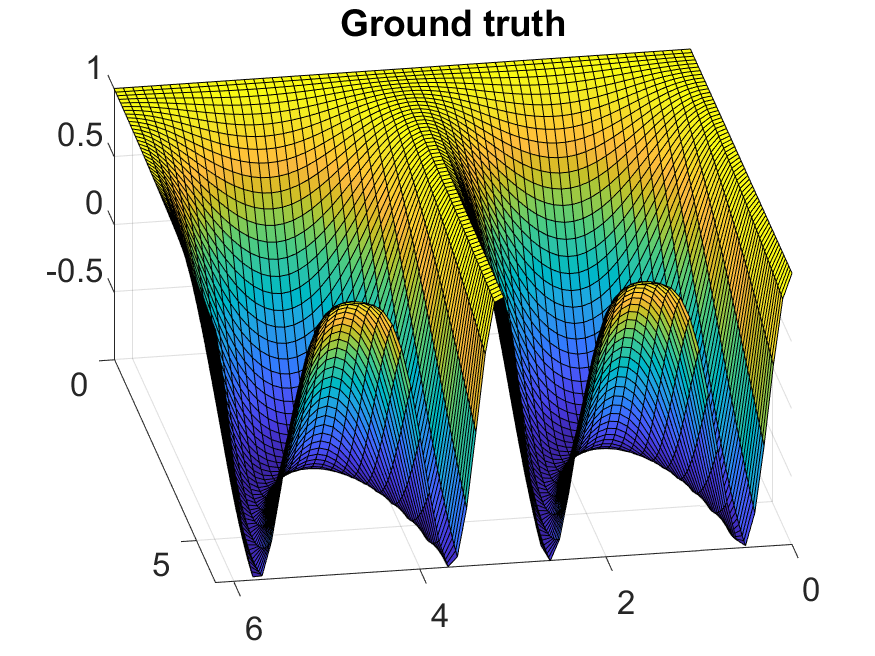}
    \includegraphics[width=0.3\textwidth]{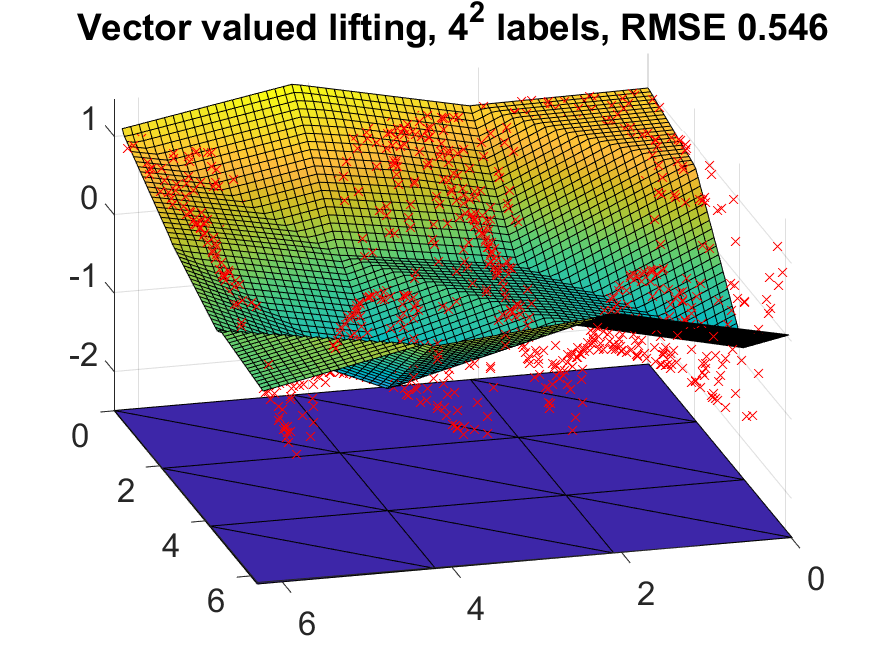}
    \includegraphics[width=0.3\textwidth]{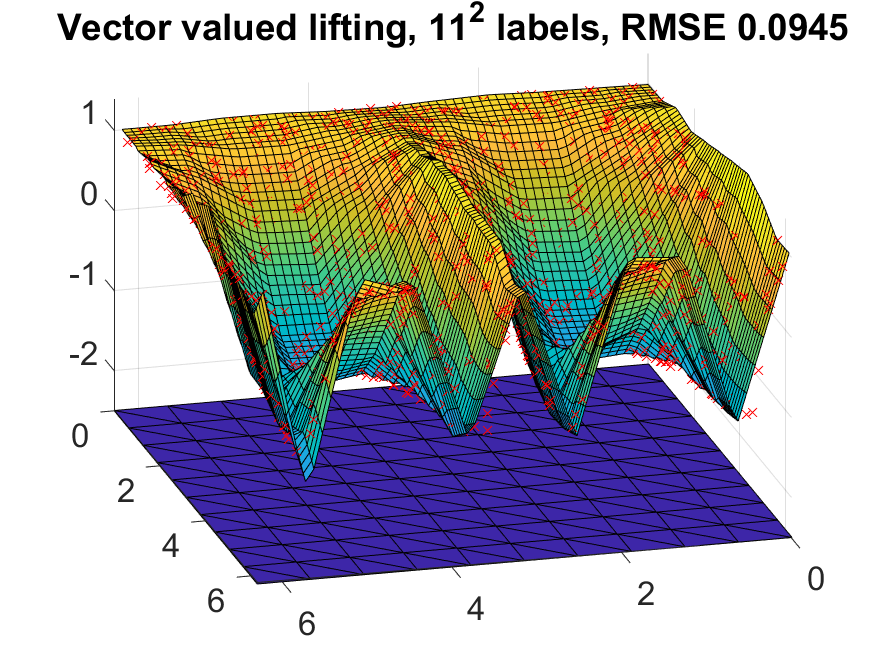}
  \end{center}
  \caption{\label{fig:2dExample}Illustrating the results of approximating the function in \eqref{eq:2dExampleFunction} with the standard network in \eqref{eq:stdnet2d} (middle row) and the architecture in \eqref{eq:lnet2d} based on lifting the input data (upper row). The red markers illustrate the training data, the surface represents the overall network function, and the RMSE measures its difference to the true underlying function \eqref{eq:2dExampleFunction}, which is shown in the bottom row on the left. Similar to the results of Figure \ref{fig:1dExample}, our lifting based architecture converges more quickly and yields a better approximation of the true underlying function (lower left) after 2000 epochs.
  The middle and right approximations in the bottom row illustrate a vector-valued lifting (see Section~\ref{sec:vec-lift-layer}) into $4^2$ (middle) and $11^2$ (right) dimensions. The latter can be trained by solving a linear system. We illustrate the triangular mesh used for the lifting below the graph of the function to illustrate that the approximation is indeed piecewise linear (as stated in Proposition~\ref{prop:lifting-leads-to-lin-spline}).
  }
\end{figure}
While the above results were expected based on the favorable theoretical properties, we now consider a more difficult test case of fitting the function 
\begin{align}
\label{eq:2dExampleFunction}
 f(x_1,x_2) = \cos(x_2\: \sin(x_1))
\end{align}
on $[0,2\pi]^2$. Note that although a 2-D input still allows for a vector-valued lifting, our goal is to illustrate that even a coordinate-wise lifting has favorable properties (beyond being able to approximate any separable function with a single layer, which is a simple extension of Corollary~\ref{cor:universalApproximation}). We therefore compare the two networks
\begin{align}
\tag{Lift-Net}
\label{eq:lnet2d}
f_{\text{Lift-Net}}(x_1,x_2) =&~ \text{fc}_1(\sigma(\text{fc}_{20}([\lift_{20}(x_1), \lift_{20}(x_2)]))), \\
\tag{Std-Net}
\label{eq:stdnet2d}
f_{\text{Std-Net}}(x_1,x_2) =&~ \text{fc}_1(\sigma(\text{fc}_{20}(\text{fc}_{40}([x_1,x_2])))) ,
\end{align} 
where the notation $[u,v]$ in the above formula denotes the concatenation of the two vectors $u$ and $v$.
The corresponding training now yields a non-convex optimization problem in both cases. As we can see in Figure~\ref{fig:2dExample} the general behavior is similar to the 1-D case: Increasing the dimensionality via lifting the input data yields faster convergence and a more precise approximation than increasing the dimensionality with a parameterized filtering. For the sake of completeness, we have included a vector-valued lifting with an illustration of the underlying 2-D triangulation in the bottom row of Figure~\ref{fig:2dExample}.

\subsection{Image Classification} \label{sec:image_classifcation}
As a real-world imaging example we consider the problem of image classification. To illustrate the behavior of our lifting layer, we use the ``Deep MNIST for expert model'' (\textit{ME-model}) by TensorFlow\footnote{\url{https://www.tensorflow.org/tutorials/layers}} as a simple standard architecture:
$$ \underset{(5\times5\times32)}{\text{Conv}} ~ \stackrel{\text{ReLU}}{\rightarrow} ~\underset{(2\times2)}{\text{Pool}}~ \rightarrow ~\underset{(5\times5\times64)}{\text{Conv}}~ \stackrel{\text{ReLU}}{\rightarrow} ~\underset{(2\times2)}{\text{Pool}}~ \rightarrow ~\underset{(1024)}{\text{FC}}~ \stackrel{\text{ReLU}}{\rightarrow} ~\underset{(n)}{\text{FC}}$$
which applies a standard ReLU activation, max pooling and outputs a final number of $n$ classes. 
In our experiments, we use an additional batch-normalization (BN) to improve the accuracy significantly, and denote the corresponding model by \textit{ME-model+BN}.

Our model is formed by replacing all ReLUs by a scaled lifting layer (as introduced in Section~\ref{sec:practicalScaledLifting}) with $L=3$, where we scaled with the absolute value $|t^i|$ of the labels to allow for a meaningful combination with the max pooling layers. We found the comparably small lifting of $L=3$ to yield the best results in (deeply) nested architectures. As our lifting layer increases the number of channels by a factor of $2$, our model has almost twice as many free parameters as the ME model.
Since this could yield an unfair comparison, we additionally include a larger model \textit{Large ME-model+BN} with twice as many convolution filters and fully-connected neurons resulting in even more free parameters than our model.

Figure~\ref{fig:classification} shows the results each of these models obtains on the image classification problems CIFAR-10 and CIFAR-100. As we can see, the favorable behavior of the synthetic experiments carried over to the exemplary architectures in image classification: Our proposed architecture based on lifting layers has the smallest test error and loss in both experiments. Both common strategies, i.e. including batch normalization and increasing the size of the model, improved the results, but even the larger of the two ReLU-bases architectures remains inferior to the lifting-based architecture. 

\begin{figure}[t]
  \begin{center}
    \begin{tabular}{cccc}
    \includegraphics[width = 0.24\textwidth]{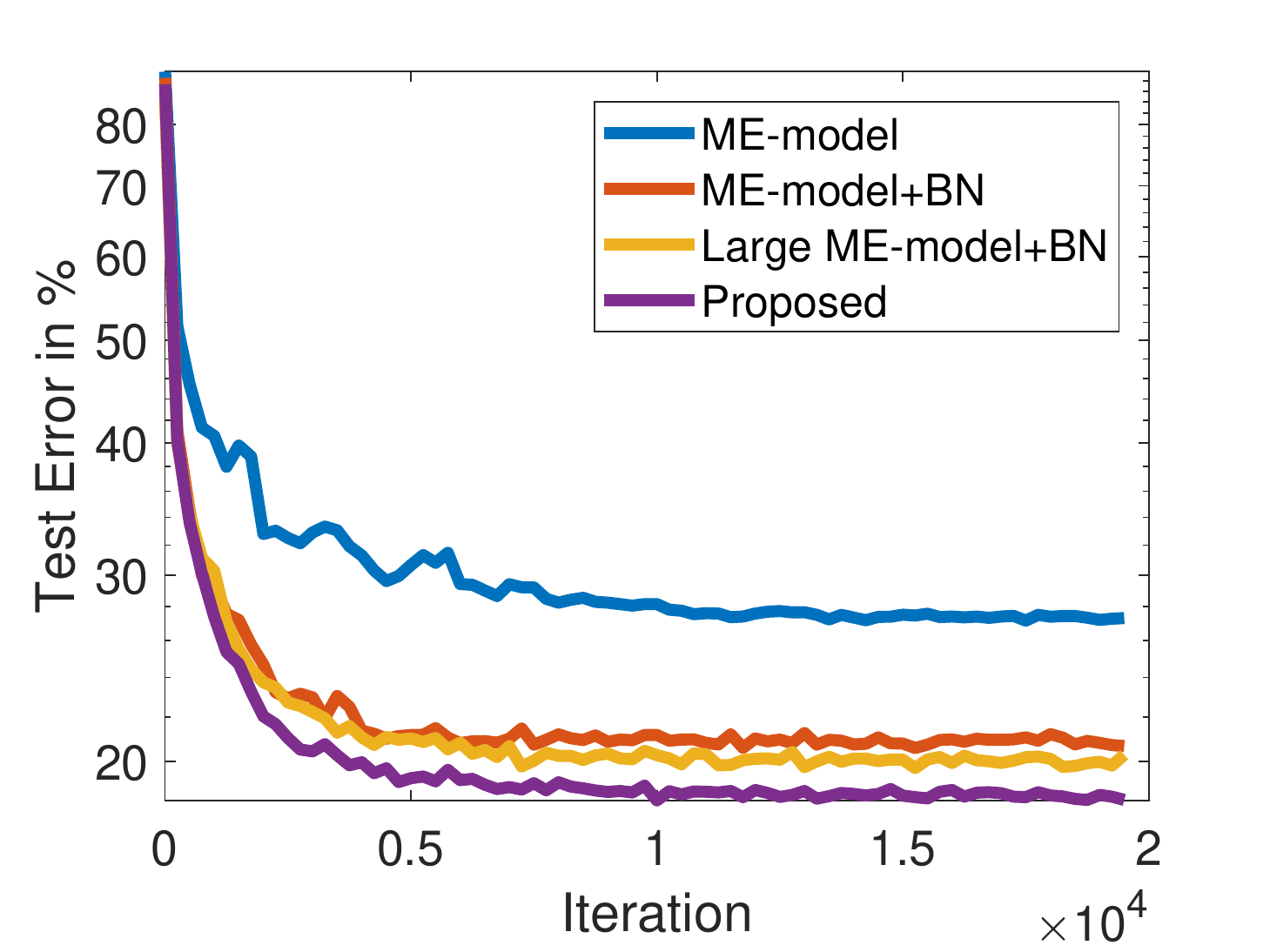} & \includegraphics[width = 0.24\textwidth]{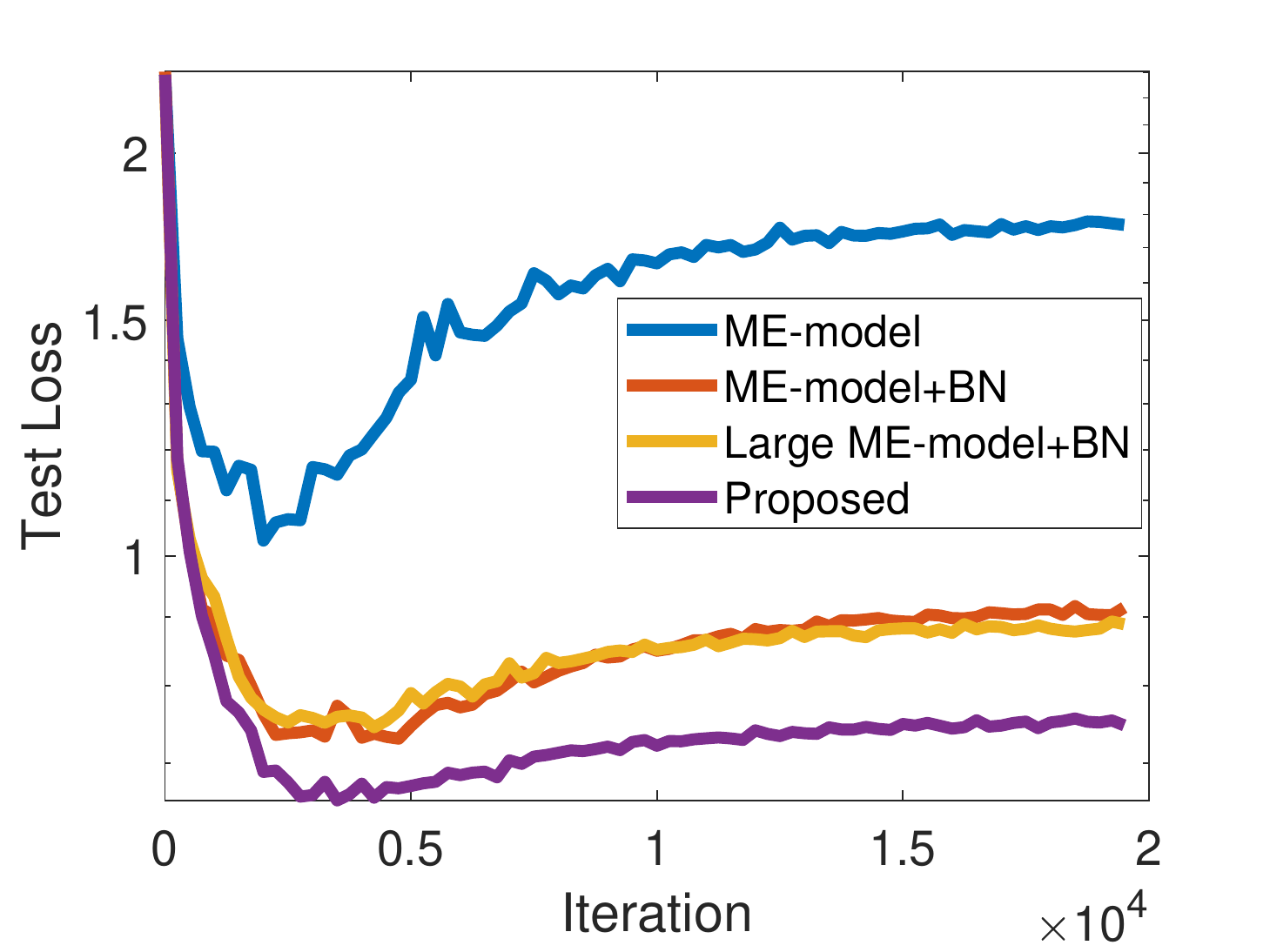} &
    \includegraphics[width = 0.24\textwidth]{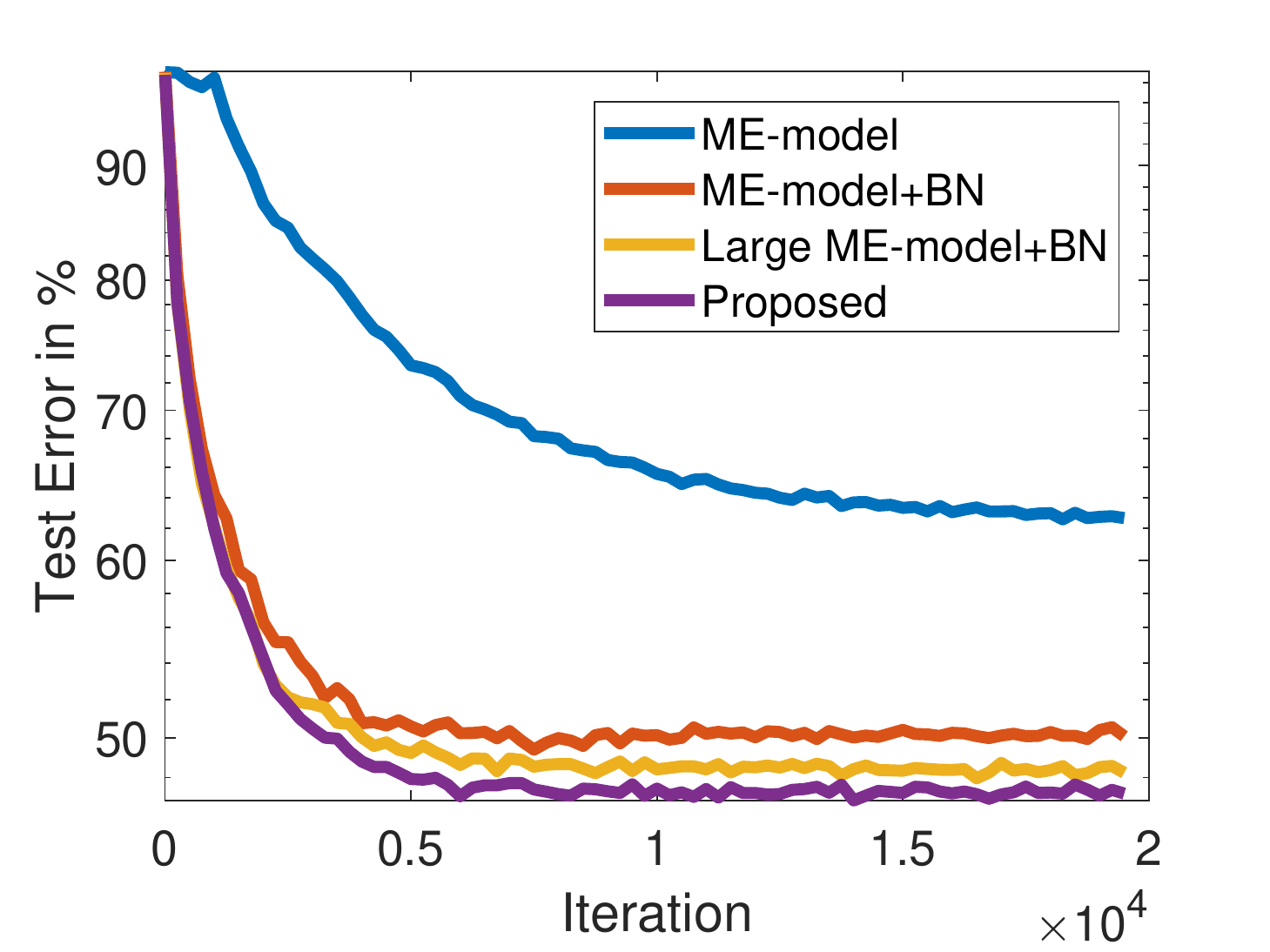} & \includegraphics[width = 0.24\textwidth]{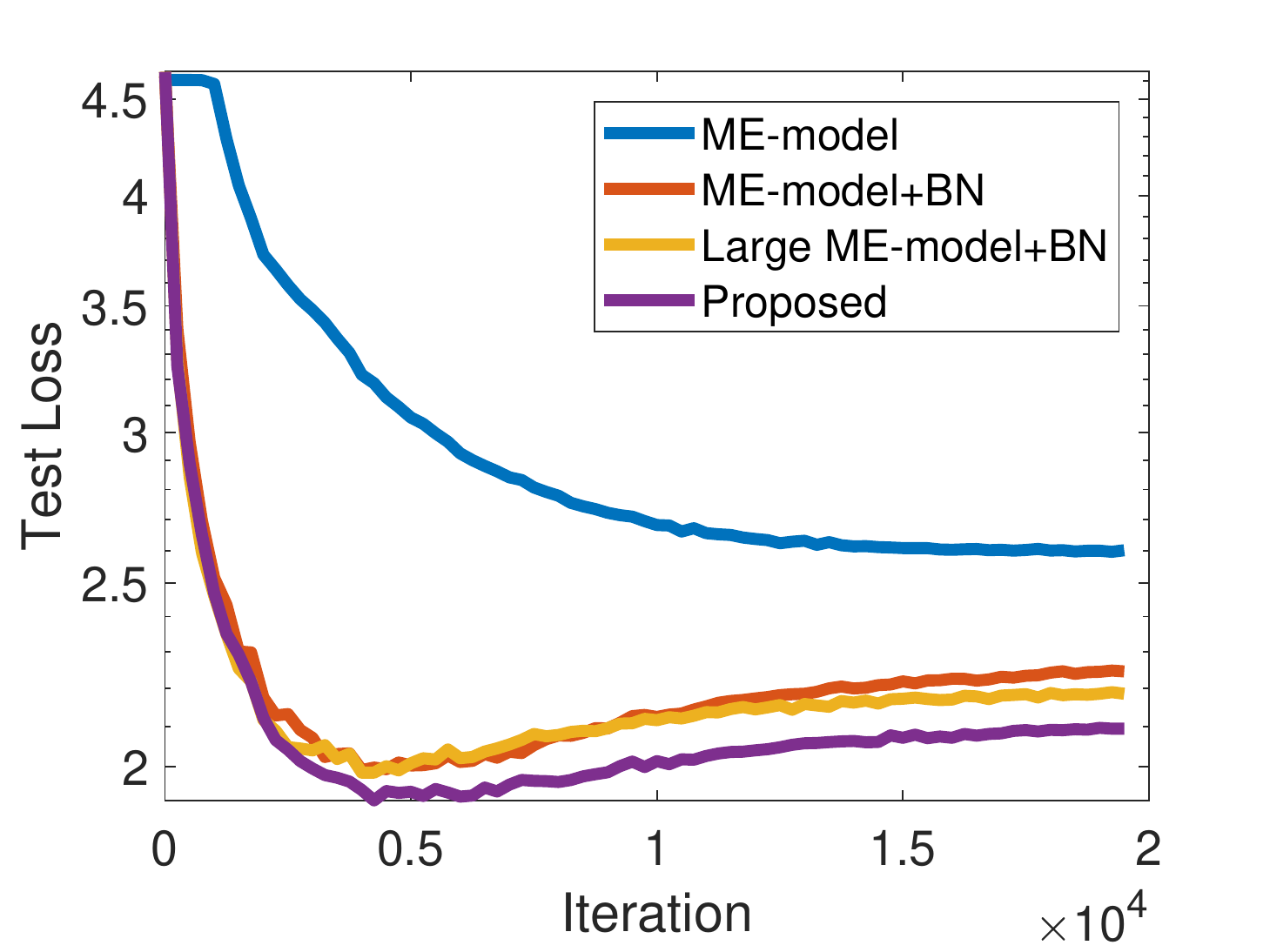}\\
    \tiny (a) CIFAR-10 Test Error & \tiny(b) CIFAR-10 Test Loss&
    \tiny(c) CIFAR-100 Test Error & \tiny(d) CIFAR-100 Test Loss
    \end{tabular}
  \end{center}
  \caption{\label{fig:classification}Comparing different approaches for image classification on CIFAR-10 and CIFAR-100. The proposed architecture with lifting layers shows a superior performance in comparison to its ReLU-based relatives in both cases. }
\end{figure}

\subsection{Maxout Activation Units}

To also compare the proposed lifting activation layer with the maxout activation, we conduct a simple MNIST image classification experiment with a fully connected one-hidden-layer architecture, using a ReLu, maxout or lifting as activations. For the maxout layer we apply a feature reduction by a factor of $2$ which has the capabilities of representing a regular ReLU and a lifting layer as in \eqref{eq:reducedLifting}. Due to the nature of the different activations - maxout applies a max pooling and lifting increases the number of input neurons in the subsequent layer - we adjusted the number of neurons in the hidden layer to make for an approximately equal and fair amount of trainable parameters.

The results in Figure \ref{fig:maxout_comparison} are achieved after optimizing a cross-entropy loss for $100$ training epochs by applying SGD with learning rate $0.01$. Particularly, each architecture was trained with the identical experimental setup. 
While both the maxout and our lifting activation yield a similar convergence behavior better than the standard ReLU, the proposed method exceeds in terms of the final lowest test error.

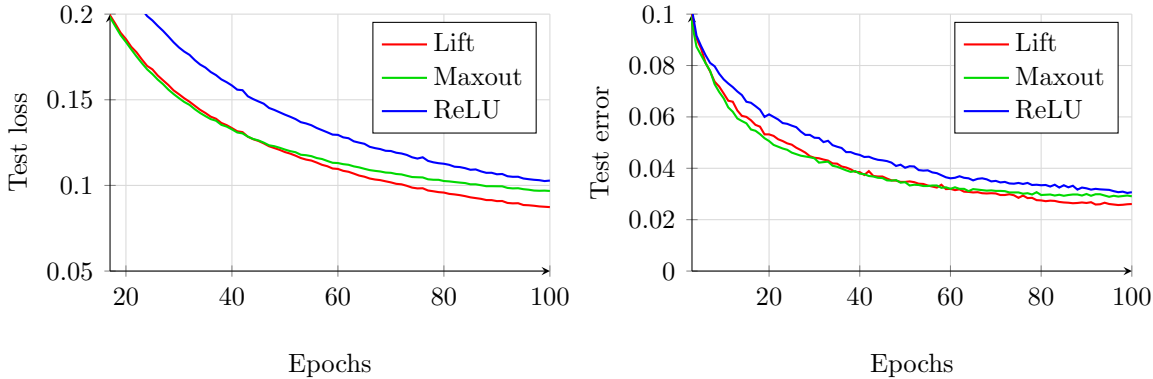
\begin{figure}[t]
  \begin{center}
    \begin{tikzpicture}
        \begin{axis}[width=0.45 \textwidth,
                     height=5cm,
                     grid=both,
                     grid style={solid,gray!30!white},
                     axis lines=left,
                     every axis plot/.append style={semithick},
                     ymin=0.05,
                     ymax=0.2,
                     xlabel={Epochs},
                     ylabel={Test loss},
                     scaled ticks=false,
                     tick label style={/pgf/number format/fixed},
                     x label style={at={(axis description cs:0.5,-0.125)},anchor=north},
                     legend cell align=left,
                     legend pos=north east,]
            \addplot [draw=red,thick] table [x=Step, y=Value, col sep=comma] {data/run_lift-tag-test_loss.csv};
            \addplot [draw=green!85!black,thick]  table [x=Step, y=Value, col sep=comma] {data/run_maxout-tag-test_loss.csv};
            \addplot [draw=blue,thick] table [x=Step, y=Value, col sep=comma] {data/run_relu-tag-test_loss.csv};
            \legend{Lift, Maxout, ReLU}
        \end{axis}
    \end{tikzpicture}
    \begin{tikzpicture}
        \begin{axis}[width=0.45 \textwidth,
                     height=5cm,
                     grid=both,
                     grid style={solid,gray!30!white},
                     axis lines=left,
                     every axis plot/.append style={semithick},
                     ymin=0.0,
                     ymax=0.1,
                     xlabel={Epochs},
                     ylabel={Test error},
                     scaled ticks=false,
                     tick label style={/pgf/number format/fixed},
                     x label style={at={(axis description cs:0.5,-0.125)},anchor=north},
                     legend cell align=left,
                     legend pos=north east,]
            \addplot [draw=red,thick] table [x=Step, y expr=1.0 - \thisrow{Value}, col sep=comma] {data/run_lift-tag-test_accuracy.csv};
            \addplot [draw=green!85!black,thick]  table [x=Step, y expr=1.0 - \thisrow{Value}, col sep=comma] {data/run_maxout-tag-test_accuracy.csv};
            \addplot [draw=blue,thick] table [x=Step, y expr=1.0 - \thisrow{Value}, col sep=comma] {data/run_relu-tag-test_accuracy.csv};
            \legend{Lift, Maxout, ReLU}
        \end{axis}
    \end{tikzpicture}
  \end{center}
  \caption{\label{fig:maxout_comparison}MNIST image classification comparison of our lifting activation with the standard ReLU and its maxout generalization. 
  The ReLU, maxout and lifting architectures ($79510$, $79010$ and $76485$ trainable parameters) achieved a best test error of $3.07\%$, $2.91\%$ and $2.61\%$, respectively. The proposed approach behaves favorably in terms of the test loss from epoch 50 on, leading to a lower overall test error after 100 epochs.}
\end{figure}

\subsection{Image Denoising}

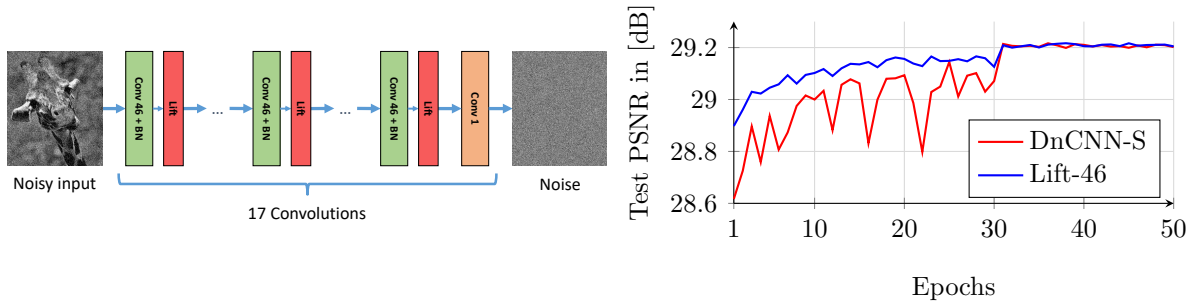
\begin{figure}[t]
    \begin{center}
    \raisebox{0.45\height}{\includegraphics[width=0.5\textwidth]{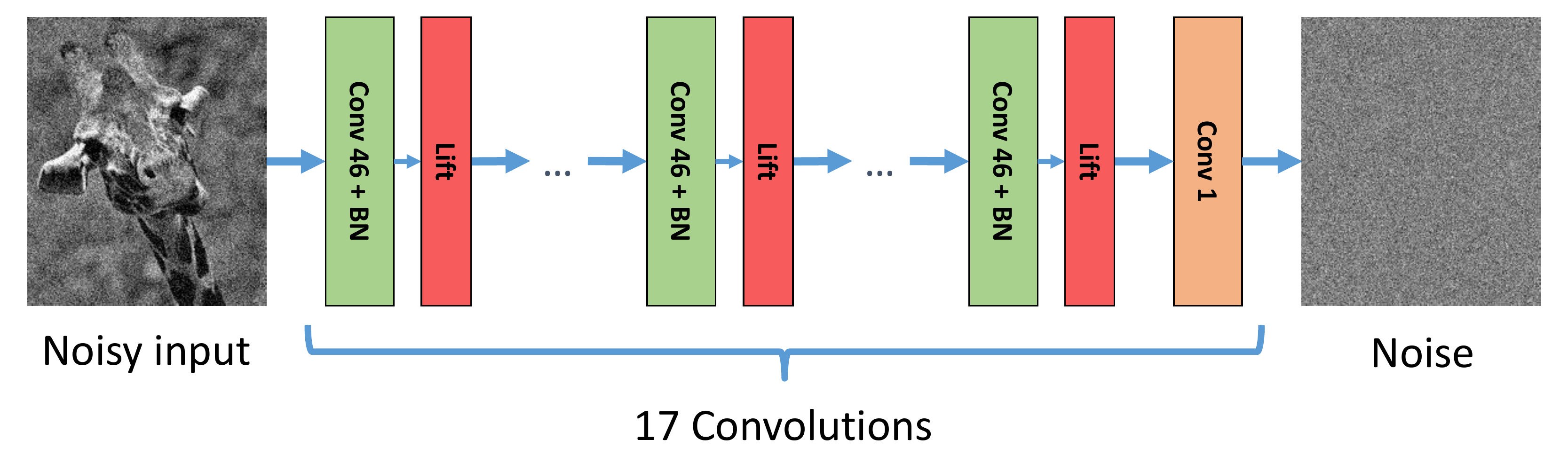}}
    \centering
    \begin{tikzpicture}
        \begin{axis}[width=0.45\textwidth,
                     height=4cm,
                     grid=both,
                     grid style={solid,gray!30!white},
                     axis lines=left,
                     every axis plot/.append style={semithick},
                     ymin=28.6,
                     ymax=29.3,
                     xtick={1,10,20,30,40,50},
                     xlabel={Epochs},
                     ylabel={Test PSNR in [dB]},
                     scaled ticks=false,
                     tick label style={/pgf/number format/fixed},
                     x label style={at={(axis description cs:0.5,-0.125)},anchor=north},
                     legend cell align=left,
                     legend pos=south east,]
            \addplot [draw=red,thick] table [x expr=\thisrow{Step} + 1, y=Value, col sep=comma] {data/run_dncnn-tag-PSNR_on_validation_data.csv};
            \addplot [draw=blue,thick] table [x=Step, y=Value, col sep=comma] {data/run_lift_net_46_17-tag-PSNR_on_validation_data.csv};
            \legend{DnCNN-S, Lift-46}
        \end{axis}
    \end{tikzpicture}
    \end{center}
    \caption{In (a) we illustrate our Lift-46 image denoising architecture which implements 16 convolution layers with 46 filters. Although its test PSNR in (b) for Gaussian noise with $\sigma = 25$ plateaus - after a learning rate decay at 30 epochs - to the same final value it generally shows a favorable and more stable behavior.}
    \label{fig:denoising_arch_psnr}
\end{figure}

\begin{table}
  \caption{\label{table:denoising_comparison}Average PSNRs in [dB] for the BSD68 dataset for different standard deviations $\sigma$ of the Gaussian noise on all of which our lifting layer based architecture is among the leading methods. Please note that (most likely due to variations in the random seeds) our reproduced DnCNN-S results are different - in the second decimal place - from the results reported in \cite{ZZCMZ17}.
  }
  \begin{center}
    \newcolumntype{Y}{>{\centering\arraybackslash}X}
    \resizebox{\textwidth}{!}{
    \begin{tabularx}{1.85\textwidth}{|Y|*{9}{Y|}}
        \hline
        \multicolumn{10}{|c|}{Reconstruction PSNR in [\textit{dB}]} \\
        \hline
        {$\sigma$} &  {Noisy} & {BM3D \cite{DFKE07}} & {WNNM \cite{GZZF14}} & {EPLL \cite{ZW11}} & {BSH12 \cite{BSH12}} & {CSF \cite{SR14}} & {TNRD \cite{CP17}} & {DnCNN-S~\cite{ZZCMZ17}}  & {Our} \\
        \hline
            15 & 24.80 & 31.07 & 31.37 & 31.21 & -     & 31.24 & 31.42 & \textbf{31.72} & \textbf{31.72} \\
            25 & 20.48 & 28.57 & 28.83 & 28.68 & 28.96 & 28.74 & 28.92 & \textbf{29.21} & \textbf{29.21} \\
            50 & 14.91 & 25.62 & 25.87 & 25.67 & 26.03 & -     & 25.97 & 26.21          & \textbf{26.23} \\
        \hline
    \end{tabularx}
    }
  \end{center}
\end{table}

To also illustrate the effectiveness of lifting layers for networks mapping images to images, we consider the problem of Gaussian image denoising. We designed the \textit{Lift-46} architecture with $16$ blocks each of which consists of $46$ convolution filters of size $3\times 3$, batch normalization, and a lifting layer with $L=3$ following the same experimental reasoning for deep architectures as in Section \ref{sec:image_classifcation}. As illustrated in Figure \ref{fig:denoising_arch_psnr}(a), a final convolutional layer outputs an image we train to approximate the residual, i.e., noise-only, image. Due to its state-of-the-art performance in image denoising we adopted the same training pipeline as for the \textit{DnCNN-S} architecture from \cite{ZZCMZ17} which resembles our Lift-46 network but implements a regular ReLU and $64$ convolution filters. The two architectures contain an approximately equal amount of trainable parameters.


Table \ref{table:denoising_comparison} compares our architecture with a variety of denoising methods most notably the DnCNN-S \cite{ZZCMZ17} and shows that we produce state-of-the-art performance for removing Gaussian noise of different standard deviations $\sigma$. In addition, the development of the test PSNR in Figure \ref{fig:denoising_arch_psnr}(b) suggests a more stable and favorable behavior of our method compared to DnCNN-S.

\section{Conclusions} \label{sec:Conclusions}

We introduced lifting layers to be used as an alternative to ReLU-type activation functions in machine learning. Opposed to the classical ReLU, liftings have a nonzero derivative almost everywhere, and can - when combined with a fully connected layer - represent any continuous piecewise linear function. We demonstrated several advantageous properties of lifting and used this technique to handle non-convex and partly flat loss functions. Based on our numerical experiments in image classification and image reconstruction, lifting layers are an attractive building block in various neural network architectures and allowed us to improve on the performance of corresponding ReLU-based architectures.

\newpage
\appendix

\section{Vector-Valued Lifting} \label{appdx:vector-val}

\paragraph{Notation for a Triangulation.}
For a non-empty, connected, and compact set $\Omega\subset\R^d$, we consider a (non-degenerate) triangulation $(T^l)_{l=1}^M$ of $\Omega$, where $T^l$ is the convex hull of $d+1$ vertices $(V^{\kappa_l(1)}, \ldots, V^{\kappa_l(d+1)})$ from the set $\mathcal V:=\set{V^{1},\ldots,V^{L}}$ of all vertices, and $\map{\kappa_l}{\set{1,\ldots,d+1}}{\set{1,\ldots,L}}$ maps indices of the vertices of $T^l$ to the corresponding indices in $\mathcal V$. The notation is illustrated in Figure~\ref{fig:vector-val-lift}.

\subsection{Definition}

\begin{definition}[Vector-Valued Lifting] \label{def:lifting-dD}
We define the lifting of a variable $x\in\Omega\subset \R^d$ from the $d$-dimensional vector space $\R^d$ with respect to an orthogonal basis $\EE:= \set{e^1,\ldots,e^L}$ of $\R^L$ and a triangulation $(T^l)_{l=1}^M\subset \R^{d}$ as a mapping $\map{\lift}{\Omega}{\R^L}$ defined by
\begin{equation} \label{eq:def-lift-dD}
  \lift(x) = \sum_{i=1}^{d+1} \lambda^l_i(x) e^{\kappa_l(i)}
  \quad  \text{with}\ l\ \text{such that}\ x\in T^l \,,
\end{equation}
where $\lambda^l_i(x)$, $i=1,\ldots,d+1$, are the barycentric coordinates of $x$ with respect to $V^{\kappa_l(1)},\ldots, V^{\kappa_l(d+1)}$.

The inverse mapping $\map{\ilift}{\R^L}{\R^d}$ is given by
\[
  \ilift(z) = \sum_{l=1}^L \frac{\scal{e^{l}}{z}}{\vnorm{e^{l}}^2} V^{l} \,.
\]
\end{definition}
\begin{example}[Scalar-Valued Lifting] \label{ex:scalar-lifting}
  For $d=1$, we obtain the scalar-valued lifting with $\Omega=[\lI,\rI]$, $\mathcal V=\set{t^1,\ldots,t^L}$, and the vertices of $T^l$ are exactly the interval borders $V^{\kappa_l(1)}=t^l$ and $V^{\kappa_l(2)}=t^{l+1}$ for $l=1,\ldots,M$ with $M=L-1$.
\end{example}
\begin{example}
  For $\Omega=[\lI,\rI]^d$, a regular grid on the rectangular domain in $\R^d$, a natural triangulation is induced by the vertices $\mathcal V:=[t^1,\ldots,t^L]^d$, $\lI=t^1<\ldots<t^L=\rI$, which implies a lifted dimension of $dL$.
\end{example}

\begin{lemma}[Sanity Check of Inversion Formula] \label{lem:inv-formula-dD}
  The mapping $\ilift$ inverts the mapping $\lift$, i.e. $\ilift(\lift(x))=x$ for $x\in \Omega\subset\R^d$.
\end{lemma}
\begin{proof}
For $x\in T^l$, using $\scal{e^l}{e^k} = 0$ for $l\neq k$ (since $\EE$ is orthogonal) , the following holds:
\[
  \ilift(\lift(x)) 
  =  \sum_{k=1}^L\sum_{i=1}^{d+1} \lambda^l_i(x) \frac{\scal{e^k}{e^{\kappa_l(i)}}}{\vnorm{e^k}^2} V^k 
  =  \sum_{i=1}^{d+1} \lambda^l_i(x) V^{\kappa_l(i)} = x \,.
\]
where the last equality uses the definition of barycentric coordinates.
\end{proof}

\subsection{Analysis}

\begin{proposition}[Prediction of a Continuous Piecewise Linear Functions] \label{prop:lifting-leads-to-lin-spline-dD}
 The composition of a fully connected layer $z\mapsto \theta z$ with $\theta\in \R^{r\times L}$, $r\in\N$, and a lifting layer, i.e.
 \begin{equation} \label{eq:predict-spline-dD-net}
    \net_{\theta}(x) := \theta \lift(x) \,,
 \end{equation}
 yields a continuous piecewise linear (PLC) function. Conversely, any PLC function with kinks on a triangulation of $\Omega$ can be expressed in the form of \eqref{eq:predict-spline-dD-net}.
\end{proposition}
\begin{proof}
  For $x\in T^l$, we have: 
  \[
      A(\lift(x)) 
      = \theta \lift(x)
      = \theta \sum_{i=1}^{d+1} \lambda^l_i(x) e^{\kappa_l(i)}
      = \sum_{i=1}^{d+1} \lambda^l_i(x) \theta e^{\kappa_l(i)} \,.
  \]
  Since $\lambda^l_i(x)$ is linear, the expression on the right coincides with the linear interpolation between the points $\theta e^{\kappa_l(i)}$, $i=1,\ldots,d+1$. Continuity follows by continuity of the expression above at the boundary of $T^l$, for each $l=1,\ldots,M$. 
  
  The converse statement follows by defining the lifting with respect to the same triangulation as the given PLC function and choosing $\theta$ such that $\net_{\theta}$ coincides with that function on the vertices. The details are analogue to the proof of Corollary~\ref{prop:overfitting-dD}.
\end{proof}

\begin{lemma}[Approximation by Continuous Piecewise Linear Functions] \label{lem:approx-cont-with-lin-spline-dD}
  Let $\map{f}{\Omega}{\R^r}$, $r\in\N$, be a continuous function with the following modulus:
  \begin{equation} \label{eq:modulus-continuity-dD}
    \omega(f,\delta):=\sup\set{\vnorm{f(x)-f(y)}\setsep \vnorm{x-y}\leq \delta,\ \forall x,y\in\Omega} \,.
  \end{equation}
  Define the continuous piecewise linear function $\map{s_f}{\Omega}{\R^r}$ on the triangulation $(T^l)_{l=1}^M$  by setting $s_f(x) = f(x)$ at all vertices $x\in \set{V^1,\ldots,V^L}$. We denote by $h^M_l$ the diameter of $T^l$, given by 
  \[
    h^M_l:=\sup\set{\vnorm{V^{\kappa_l(i)} - V^{\kappa_l(i^\prime)}}\setsep i,i^\prime=1,\ldots,d+1} \,,
  \]
  and set $h^M:=\max_{l=1,\ldots,M} h^M_l$, which is finite. Then
  \[
    \sup_{x\in\Omega} \vnorm{f(x) - s_f(x)} \leq \omega(f,h^M) 
  \]
  and the right hand side vanishes for $h^M\dto 0$. 
\end{lemma}
\begin{proof}
  For $x\in T^l$, let $s_f$ be given by $s_f(x) = \sum_{i=1}^{d+1} \lambda^l_i(x) f(V^{\kappa_l(i)})$ with $\lambda^l_i(x)\in [0,1]$ and $\sum_{i=1}^{d+1} \lambda^l_i(x) = 1$. Note that $s_f$ is uniquely defined.   
  We conclude:
  \[
    \vnorm{f(x) - s_f(x)} 
    \leq \vnorm{\sum_{i=1}^{d+1} \lambda^l_i(x) \Big( f(x) - f(V^{\kappa_l(i)}) \Big) } 
    \leq \sup_{y\in T^l} \vnorm{f(x) - f(y)} 
    \leq \omega(f, h^M_l)  \,.
  \]
  As $\Omega$ is compact, $f$ is uniformly continuous, which, together with $\omega(f, h^M_l)\leq \omega(f, h^M)$, implies that the right hand side vanishes for $h^M\dto 0$.
\end{proof}
\begin{example}
  Consider a (locally) Lipschitz continuous function $\map{f}{\Omega}{\R^r}$. By compactness of $\Omega$, the function $f$ is actually globally Lipschitz continuous on $\Omega$ with a constant $m$, which implies $\omega(f,\delta)\leq \delta m$, since $\vnorm{f(x) - f(y)} \leq m\vnorm{x-y}$. 
\end{example}
\begin{corollary}[Prediction of Continuous Functions]  \label{cor:universalApproximation-dD}
  Any continuous function $\map{f}{\Omega}{\R^r}$, $r\in\N$, can be represented arbitrarily accurate with a network architecture $\net_{\theta}(x) = \theta \lift(x)$ for sufficiently large $L$ and $\theta\in \R^{r\times L}$.
\end{corollary}
\begin{proof}
  Combine Proposition~\ref{prop:lifting-leads-to-lin-spline-dD} with Lemma~\ref{lem:approx-cont-with-lin-spline-dD}.
\end{proof}

\begin{corollary}[Overfitting]  \label{prop:overfitting-dD}
  Let $(x_i,y_i)$ be training data in $\Omega\times \R^r$, $i=1,\hdots,N$, $x_i\neq x_j$ for $i\neq j$. If $L=N$ and $V^i = x_i$, there exists $\theta\in\R^{r\times L}$ such that $\net_\theta(x):=\theta\lift(x)$ is exact at all data points $x=x_i$, i.e. $\net_\theta(x_i) = y_i$, for all $i=1,\ldots,N$.
\end{corollary}
\begin{proof}
  Since $x_i=V^i$, \eqref{eq:def-lift-dD} shows that $\lambda^l_i(x)=1$ and $\lambda^l_j(x)=0$ for $j\neq i$. Therefore, we have $\theta\lift(x_i) = \theta e^{\kappa_l(i)}$. Denote by $E\in \R^{L\times L}$ the matrix with columns given by $e^1,\ldots, e^L$, and $y\in \R^{r\times L}$ the matrix with columns $y_1,\ldots,y_L$. Since $\EE$ is a basis, the matrix $E$ is non-singular, and we may determine $\theta$ uniquely by solving the following linear system of equations $\theta E = y$, which concludes the statement. 
\end{proof}

\begin{proposition}[Convexity of a simple Regression Problem] \label{prop:convex-regression-prob-dD}
  Let $(x_i,y_i) \in \Omega \times \R^r$ be training data, $i=1,\hdots,N$. Then, the solution of the problem
  \begin{equation}
    \label{eq:linearSpline-dD}
    \min_{\theta\in\R^{r \times L}}\, \sum_{i=1}^N\mathcal{L}(\theta\lift(x_i); y_i)
  \end{equation}
  yields the best continuous piecewise linear fit of the training data with respect to the loss function $\mathcal{L}$. In particular, if $\mathcal{L}$ is convex, then \eqref{eq:linearSpline-dD} is a convex optimization problem. 
\end{proposition}
\begin{proof}
  Proposition~\ref{prop:lifting-leads-to-lin-spline-dD} shows that $x\mapsto \theta\lift(x)$ is a continuous piecewise linear function. Obviously, $\theta\mapsto \theta\lift(x_i)$ is linear, hence composed with a convex loss function, \eqref{eq:linearSpline-dD} is a convex optimization problem. 
\end{proof}

\section{Lifting the Output} \label{appdx:lift-output}

\begin{lemma}[{Characterization of the Range of $\lift$}] \label{lem:char-range-lift-dD}
  The range of the mapping $\lift$ is given by
  \begin{equation} \label{eq:def-range-lift-dD}
    \im(\lift) = \set{z\in [0,1]^L \setsep 
                  \begin{matrix} z=\sum_{l=1}^L z_l e^l\,,\ 
                  \exists l\in\set{1,\ldots,M}\colon \sum_{i=1}^{d+1} z_{\kappa_l(i)} = 1 \\ \text{and}\
                  \forall k\not\in \im(\kappa_l)\colon z_k=0
                  \end{matrix}
                  } 
  \end{equation}
  and the mapping $\lift$ is a bijection between $\Omega$ and $\im(\lift)$ with inverse $\ilift$.
\end{lemma}
\begin{proof}
  Let $z\in [0,1]^L$ be given by $z=\sum_{l=1}^L z_l e^l$ and there exists exactly one index $l$ such that $\sum_{i=1}^{d+1} z_{\kappa_l(i)} = 1$ and, for all $k\not\in \im(\kappa_l)$, we have $z_k=0$. The point $x$ given by $x=\sum_{i=1}^{d+1} z_{\kappa_l(i)} V^{\kappa_l(i)}$ maps to $z$ via $\lift$. Obviously $x\in T^l$, which implies that $\lift(x)=\sum_{i=1}^{d+1} \lambda^l_i(x)e^{\kappa_l(i)}$ and, by the uniqueness of barycentric coordinates, $\lambda^l_i = z_{\kappa_l(i)}$. Moreover \eqref{eq:def-lift-dD} implies for $k\not\in\im(\kappa_l)$ that $z_k=0$. We conclude that the set on right hand side of \eqref{eq:def-range-lift-dD} is included in $\im(\ell)$. By the definition in \eqref{eq:def-lift-dD}, it is clear that $z=\ell(x)$ for $x\in \Omega$ satisfies the condition for belonging to the set on the right hand side of \eqref{eq:def-range-lift-dD}, which implies their equality.

  In order to prove the bijection, injectivity remains to show. This is proved as follows: For $x,x^\prime$ such that $\ell(x)=\ell(x^\prime)$, the definition in \eqref{eq:def-lift-dD} requires that $x,x^\prime$ lie in the same $T^l$, and the property of a basis implies $\lambda_l(x) = \lambda_l(x^\prime)$, which implies that $x=x^\prime$ holds. Finally, the proof of $\lift(\ilift(z))=z$ for $z\in \im(\lift)$ follows similar arguments as the first part of this proof.
\end{proof}

\begin{lemma}[{Convex Relaxation of the Range of $\lift$}] \label{lem:conv-relax-dD}
  The set $\mathcal C$ given by 
  \begin{equation} \label{eq:range-conv-relax-dD}
    \mathcal C:= \set{z\in [0,1]^L\setsep z=\sum_{l=1}^L z_l e^l\,,\ \sum_{l=1}^L z_l = 1}
  \end{equation}
  is the convex hull of $\im(\lift)$.
\end{lemma}
\begin{proof}
  We make the abbreviation $\mathcal I=\im(\lift)$. Obviously, $\mathcal I\subset \mathcal C$ and $\mathcal C$ is convex. Therefore, we need to show that $\mathcal C$ is the smallest convex set that contains $\mathcal I$. 

The convex hull $\conv\mathcal I$ of $\mathcal I$ consists of all convex combinations of points in $\mathcal I$. By the characterization of $\mathcal I$ in \eqref{eq:def-range-lift-dD}, it is clear that $\set{e^1,\ldots,e^L}\subset \mathcal I$. Moreover, $\mathcal C \subset \conv\set{e^1,\ldots, e^L}\subset\conv\mathcal I$ holds, thus, $\mathcal I \subset \mathcal C$ already implies that $\mathcal C = \conv\mathcal I$, as the convex hull is the smallest convex set containing $\mathcal I$. 
\end{proof}

\paragraph{Proof of Proposition~\ref{prop:lift-output}.} Proposition~\ref{prop:lift-output} requires only the 1D-setting of Lemma~\ref{lem:char-range-lift-dD} and~\ref{lem:conv-relax-dD} above. For convenience of the reader, we copy the statement of Proposition~\ref{prop:lift-output} here and proof it.
\begin{proposition} \label{appdx:prop:lift-output}
  Let $(x_i,y_i)\in [\lI,\rI]\times [\lI_y,\rI_y]$ be training data, $i=1,\ldots,N$. Moreover, let $\lift_y$ be a lifting of the common image $[\lI_y,\rI_y]$ of the loss functions $\loss_{y_i}$, $i=1,\ldots,N$, and $\lift_x$ is the lifting of the domain of $\loss_{y}$. Then
  \begin{equation} \label{appdx:eq:conv-rel-of-non-convex}
    \min_{\theta } \, \sum_{i=1}^N \ilift_y(\theta \lift_x(x_i)) \quad\st\  \theta_{p,q} \geq 0,\, \sum_{p=1}^{L_y} \theta_{p,q} = 1,\, 
    \begin{cases}
      \forall p=1,\ldots,L_y\,, \\ 
      \forall q=1,\ldots,L_x \,. 
    \end{cases}
  \end{equation}
is a convex relaxation of the (non-convex) loss function, and the constraints guarantee that $\theta \lift_x(x_i)\in \conv(\im(\lift_y))$.

  The objective in \eqref{eq:conv-rel-of-non-convex} is linear (w.r.t. $\theta$) and can be written as
  \begin{equation} \label{appdx-eq:lifted-output-cost-matrix}
      \sum_{i=1}^N \ilift_y(\theta \lift_x(x_i)) 
      = \sum_{i=1}^N \sum_{p=1}^{L_y} \sum_{q=1}^{L_x} \theta_{p,q} \lift_x(x_i)_q t^p_y 
      =:  \sum_{p=1}^{L_y} \sum_{q=1}^{L_x}  c_{p,q}  \theta_{p,q} 
  \end{equation}
  where $c:=\sum_{i=1}^N t_y \lift_x(x_i)^\top$, with $t_y := (t^1_y, \ldots, t^{L_y}_y)^\top$, is the cost matrix for assigning the loss value $t_y^p$ to the inputs $x_i$.

  Moreover, the closed-form solution of \eqref{eq:conv-rel-of-non-convex} is given for all $q=1,\ldots, L_x$ by $\theta_{p,q} = 1$, if the index $p$ minimizes $c_{p,q}$, and $\theta_{p,q}=0$ otherwise.

\end{proposition}
\begin{proof}
\eqref{appdx:eq:conv-rel-of-non-convex} is obviously a convex problem, which was generated by relaxing the constraint set $\im(\lift_y)$ using Lemma~\ref{lem:conv-relax-dD}. Restricting $\theta$ to $\im(\lift_y)$ yields, obviously, a piecewise linear approximation of the true loss $\loss_y$. 

Since $z:=\lift_x(x_i)\in \im(\lift_x)$ satisfies the condition in \eqref{eq:def-range-lift-dD} and in particular the condition in \eqref{eq:range-conv-relax-dD}, we conclude that
\[
  \forall p\colon\ (\theta z)_p \geq 0 
  \quad\text{and}\quad
  \sum_{p=1}^{L_y} (\theta z)_p 
  = \sum_{p=1}^{L_y} \sum_{q=1}^{L_x} \theta_{p,q} z_q 
  = \sum_{q=1}^{L_x} z_q = 1 \,,
\]
which shows that $\theta \lift_x(x_i)\in \conv(\im(\lift_y))$.

The linearity of the objective in \eqref{appdx:eq:conv-rel-of-non-convex} is obvious, and so is \eqref{appdx-eq:lifted-output-cost-matrix}. Moreover, using the linear expression in \eqref{appdx-eq:lifted-output-cost-matrix}, clearly, the loss can be minimized by independently minimizing the cost for each $q=1,\ldots,L_x$, as the constraints couple the variables only along the $p$-dimension. For each $q$, the cost is minimized by searching the smallest entry in the cost matrix along the $p$-dimension, which verifies the closed-form solution of \eqref{appdx:eq:conv-rel-of-non-convex}. 
\end{proof}

{\small
\bibliographystyle{ieee}
\bibliography{references}

\begin{thebibliography}{10}\itemsep=-1pt

\bibitem{BSH12}
H.~C. Burger, C.~J. Schuler, and S.~Harmeling.
\newblock Image denoising: Can plain neural networks compete with {BM3D}?
\newblock In {\em International Conference on Computer Vision (ICCV)}, pages
  2392--2399, 2012.

\bibitem{CP17}
Y.~Chen and T.~Pock.
\newblock Trainable nonlinear reaction diffusion: A flexible framework for fast
  and effective image restoration.
\newblock {\em IEEE Transactions on Pattern Analysis and Machine Intelligence},
  39(6):1256--1272, 2017.

\bibitem{CUH15}
D.~Clevert, T.~Unterthiner, and S.~Hochreiter.
\newblock Fast and accurate deep network learning by exponential linear units
  ({ELUs}).
\newblock {\em Computing Research Repository (CoRR)}, abs/1511.07289, 2015.

\bibitem{Cybenko1989}
G.~Cybenko.
\newblock Approximation by superpositions of a sigmoidal function.
\newblock {\em Mathematics of Control, Signals and Systems}, 2(4):303--314,
  1989.

\bibitem{DFKE07}
K.~Dabov, A.~Foi, V.~Katkovnik, and K.~Egiazarian.
\newblock Image denoising by sparse {3-D} transform-domain collaborative
  filtering.
\newblock {\em IEEE Transactions on Image Processing}, 16(8):2080--2095, Aug.
  2007.

\bibitem{DBBNG01}
C.~Dugas, Y.~Bengio, F.~B{\'e}lisle, C.~Nadeau, and R.~Garcia.
\newblock Incorporating second-order functional knowledge for better option
  pricing.
\newblock In {\em Advances in Neural Information Processing Systems (NIPS)},
  pages 451--457, Cambridge, MA, USA, 2001. MIT Press.

\bibitem{GBC16}
I.~Goodfellow, Y.~Bengio, and A.~Courville.
\newblock {\em Deep Learning}.
\newblock The MIT Press, 2016.

\bibitem{GWMC+13}
I.~Goodfellow, D.~Warde-Farley, M.~Mirza, A.~Courville, and Y.~Bengio.
\newblock Maxout networks.
\newblock In {\em International Conference on Machine Learning (ICML)}, pages
  1319--1327, 2013.

\bibitem{GZZF14}
S.~Gu, L.~Zhang, W.~Zuo, and X.~Feng.
\newblock Weighted nuclear norm minimization with application to image
  denoising.
\newblock In {\em International Conference on Computer Vision and Pattern
  Recognition (CVPR)}, pages 2862--2869, 2014.

\bibitem{HZRS15}
K.~He, X.~Zhang, S.~Ren, and J.~Sun.
\newblock Delving deep into rectifiers: Surpassing human-level performance on
  imagenet classification.
\newblock In {\em International Conference on Computer Vision (ICCV)}, pages
  1026--1034, 2015.

\bibitem{IS15}
S.~Ioffe and C.~Szegedy.
\newblock Batch normalization: Accelerating deep network training by reducing
  internal covariate shift.
\newblock In {\em International Conference on Machine Learning (ICML)}, pages
  448--456, 2015.

\bibitem{JKRL09}
K.~Jarrett, K.~Kavukcuoglu, M.~Ranzato, and Y.~LeCun.
\newblock What is the best multi-stage architecture for object recognition?
\newblock In {\em International Conference on Computer Vision and Pattern
  Recognition (CVPR)}, pages 2146--2153, 2009.

\bibitem{KUMH17}
G.~Klambauer, T.~Unterthiner, A.~Mayr, and S.~Hochreiter.
\newblock Self-normalizing neural networks.
\newblock {\em Advances in Neural Information Processing Systems (NIPS)}, 2017.

\bibitem{LMMLC16}
E.~Laude, T.~M{\"o}llenhoff, M.~Moeller, J.~Lellmann, and D.~Cremers.
\newblock Sublabel-accurate convex relaxation of vectorial multilabel energies.
\newblock In B.~Leibe, J.~Matas, N.~Sebe, and M.~Welling, editors, {\em
  European Conference on Computer Vision (ECCV)}, pages 614--627, Cham, 2016.
  Springer International Publishing.

\bibitem{LLPS93}
M.~Leshno, V.~Lin, A.~Pinkus, and S.~Schocken.
\newblock Multilayer feedforward networks with a nonpolynomial activation
  function can approximate any function.
\newblock {\em Neural Networks}, 6(6):861--867, 1993.

\bibitem{MHN13}
A.~Maas, A.~Hannun, and A.~Ng.
\newblock Rectifier nonlinearities improve neural network acoustic models.
\newblock In {\em International Conference on Machine Learning (ICML)}, 2013.

\bibitem{MLMLC16}
T.~M\"{o}llenhoff, E.~Laude, M.~Moeller, J.~Lellmann, and D.~Cremers.
\newblock Sublabel-accurate relaxation of nonconvex energies.
\newblock In {\em International Conference on Computer Vision and Pattern
  Recognition (CVPR)}, 2016.

\bibitem{MPCB14}
G.~Mont\'{u}far, R.~Pascanu, K.~Cho, and Y.~Bengio.
\newblock On the number of linear regions of deep neural networks.
\newblock In {\em Advances in Neural Information Processing Systems (NIPS)},
  pages 2924--2932, Cambridge, MA, USA, 2014. MIT Press.

\bibitem{PCBC10}
T.~Pock, D.~Cremers, H.~Bischof, and A.~Chambolle.
\newblock Global solutions of variational models with convex regularization.
\newblock {\em SIAM Journal on Imaging Sciences}, 3(4):1122--1145, 2010.

\bibitem{SR14}
U.~Schmidt and S.~Roth.
\newblock Shrinkage fields for effective image restoration.
\newblock In {\em International Conference on Computer Vision and Pattern
  Recognition (CVPR)}, pages 2774--2781, 2014.

\bibitem{XWCL15}
B.~Xu, N.~Wang, T.~Chen, and M.~Li.
\newblock Empirical evaluation of rectified activations in convolutional
  network.
\newblock In {\em International Conference on Machine Learning (ICML)}, 2015.
\newblock Deep Learning Workshop.

\bibitem{ZZCMZ17}
K.~Zhang, W.~Zuo, Y.~Chen, D.~Meng, and L.~Zhang.
\newblock Beyond a gaussian denoiser: Residual learning of deep cnn for image
  denoising.
\newblock {\em IEEE Transactions on Image Processing}, 26(7):3142--3155, 2017.

\bibitem{ZW11}
D.~Zoran and Y.~Weiss.
\newblock From learning models of natural image patches to whole image
  restoration.
\newblock In {\em International Conference on Computer Vision (ICCV)}, pages
  479--486, 2011.

\end{thebibliography}
}

\end{document}